\documentclass[letterpaper]{article} 
\usepackage{aaai2027}  
\nocopyright  
\usepackage[hyphens]{url}  
\usepackage{graphicx} 
\usepackage{natbib}  
\usepackage{caption} 
\usepackage{booktabs}
\usepackage{multirow}
\usepackage{array}
\usepackage{tabularx}
\usepackage{amsmath,amssymb,amsfonts,amsthm}
\usepackage{mathtools}
\usepackage{listings}
\usepackage{tikz}
\usetikzlibrary{arrows.meta,backgrounds,calc,fit,positioning}

\definecolor{CodeKeyword}{RGB}{0,76,153}
\definecolor{CodeString}{RGB}{120,59,0}
\definecolor{CodeComment}{gray}{0.35}
\definecolor{CodeNumber}{gray}{0.35}
\definecolor{CodeRule}{gray}{0.55}
\definecolor{TransportInk}{RGB}{36,45,58}
\definecolor{TransportBlue}{RGB}{32,102,170}
\definecolor{TransportTeal}{RGB}{0,128,145}
\definecolor{TransportGreen}{RGB}{51,125,79}
\definecolor{TransportAmber}{RGB}{181,120,22}
\definecolor{TransportRed}{RGB}{172,63,55}
\definecolor{TransportPanel}{RGB}{246,248,251}

\lstdefinestyle{pythoncompact}{
  language=Python,
  basicstyle=\ttfamily\scriptsize,
  keywordstyle=\color{CodeKeyword},
  stringstyle=\color{CodeString},
  commentstyle=\color{CodeComment},
  numbers=left,
  stepnumber=1,
  numberstyle=\tiny\ttfamily\color{CodeNumber},
  numbersep=6pt,
  numberblanklines=false,
  frame=tb,
  framerule=0.4pt,
  rulecolor=\color{CodeRule},
  columns=fullflexible,
  keepspaces=true,
  showstringspaces=false,
  tabsize=2,
  breaklines=true,
  breakatwhitespace=true,
  xleftmargin=2.1em,
  framexleftmargin=1.9em,
  captionpos=t,
  aboveskip=0pt,
  belowskip=0pt
}

\title{Z-Loss Backward Geometry in Dense Output Heads and Sparse Routers}
\author{Bum Jun Kim\thanks{Corresponding author}}
\affiliations{
    The University of Tokyo\\
    bumjun.kim@weblab.t.u-tokyo.ac.jp
}

\newcommand{\R}{\mathbb{R}}
\newcommand{\E}{\mathbb{E}}
\newcommand{\softmax}{\operatorname{softmax}}

\newcommand{\sg}{\operatorname{sg}}
\newcommand{\vecop}{\operatorname{vec}}

\newcommand{\norm}[1]{\left\lVert #1 \right\rVert}
\newcommand{\inner}[2]{\left\langle #1,#2 \right\rangle}
\newcommand{\Eqref}[1]{Eq.~\ref{#1}}
\newcolumntype{Y}{>{\raggedright\arraybackslash}X}

\theoremstyle{plain}
\newtheorem{proposition}{Proposition}
\newtheorem{corollary}{Corollary}
\theoremstyle{remark}
\newtheorem{remark}{Remark}
\theoremstyle{plain}

\begin{document}
\maketitle

\begin{abstract}
	Z-loss has been widely applied to the logits of language-model output heads and sparse mixture-of-experts (MoE) routers. Z-loss constrains the softmax log-normalizers of these output heads and routers, thereby limiting large-logit excursions, reducing finite-precision roundoff exposure, and avoiding training-loss divergence. These use cases arise in modern Transformer settings where large-vocabulary softmax heads, top-$k$ routing, fused losses, and mixed-precision optimizers interact. Z-loss has typically been understood only as a scalar penalty on the log-normalizer. This paper instead analyzes Z-loss from a backward-pass perspective, focusing on the gradients produced by the Z-loss penalty. The logit-space gradient, which we call the backward source, is injected at the logit boundary of the Z-loss branch of backpropagation; consequently, the backward source's effect depends on the architecture and implementation through which the gradient is transported. We develop a backward-transport view for Z-loss that separates the source's scalar amplitude and softmax shape from the transport factors. These factors include common-shift coordinates, tied-embedding pathways, output-to-hidden gain, fused-loss source consistency, optimizer-facing updates, and top-$k$ router reduction scale. These diagnostics show that nearly identical forward Z-loss values can coexist with distinct logit-space Z-loss gradients and, after architectural and optimizer transport, distinct parameter updates. The transport diagnostics also explain why raw-logit Z-loss can reduce scalar tails without changing output-to-hidden gain and why active-route reductions alter the effective router coefficient. Across evaluations of models in the GPT-2 and Pythia families on WikiText-103 and FineWeb-Edu, architecture-aware variants reduce backward-geometry tails while maintaining comparable validation perplexity in low-coefficient regimes. The empirical suite includes continued pretraining, mixed-precision and fused-kernel audits, Adam-state stress tests, and end-to-end MoE training. In stress regimes, the same variants expose explicit trade-offs between validation quality and both update-tail behavior and mixed-precision headroom. Together, these results yield a transport-aware framework for measuring, reporting, and intervening on the Z-loss logit-gradient source, its transport path, and the resulting update tails.
\end{abstract}

\section{Introduction}

Modern Transformer training often includes auxiliary loss terms intended to constrain final-softmax and router log-normalizers. In dense output heads, the motivation is to limit output-logit or common-shift excursions and preserve finite-precision numerical headroom. In sparse routers, the motivation also includes reducing roundoff-sensitive routing and training-loss divergence. A prominent example is the log-normalizer penalty commonly called Z-loss, which has been widely used in modern Transformer training. Mesh TensorFlow and Pathways Language Model (PaLM) apply Z-loss to final-softmax logits \citep{shazeer2018mesh,chowdhery2023palm}. PaLM reports that keeping the softmax log-normalizer near the target increases training stability \citep{chowdhery2023palm}. Prior work accordingly uses logit-control methods as training-stability measures. Router Z-loss keeps gating logits small, reduces low-precision roundoff error, and prevents training-loss divergence \citep{zoph2022stmoe}, while output-embedding centering removes common-shift degrees of freedom associated with output-logit divergence \citep{oec2026}. Despite the broad use of Z-loss, prior work has typically presented Z-loss only as a scalar penalty on the log-normalizer.

This paper formulates Z-loss as architecture-dependent backward transport and uses this gradient perspective to analyze and diagnose how Z-loss fundamentally affects training behavior. We develop diagnostics for common backward-transport phenomena involving scalar source amplitude, source shape, output-to-hidden gain, pathway coupling, and router effective scale. Using the standard shift invariance of softmax, we derive common-shift sensitivity of Z-loss and output-head centering that preserves cross-entropy (CE). We also derive tied-embedding gradient decomposition and top-$k$ router scale distortion. Our empirical evaluation uses pretrained Transformer language models \citep{vaswani2017attention} and text from WikiText-103 and FineWeb-Edu. The model set includes four models from the Generative Pre-trained Transformer 2 (GPT-2) family: GPT-2, DistilGPT-2, GPT-2 Medium, and GPT-2 Large. The same framework is tested with matched dense continued pretraining, effectively unclipped high-coefficient stress sweeps, static mixed-precision endpoints, low-rank output-to-hidden gain spectral audits, optimizer-facing update audits, and end-to-end top-$k$ mixture-of-experts (MoE) language-model training. The implementation study additionally audits a row-wise fused CE+Z-loss graphics processing unit (GPU) kernel implemented in \mbox{Triton}, a language and a compiler for writing tiled GPU kernels \citep{tillet2019triton}.

\paragraph{Contributions.} Our primary contribution is a backward-transport framework that treats the Z-loss logit gradient as a source whose optimizer-facing effect depends on architecture and implementation, rather than on the scalar penalty alone. We derive diagnostics for common shift, tied pathways, output-to-hidden gain, fused-source consistency, and router reduction scale, together with coordinate-specific probes and interventions: centering, factorization, gain-aware weighting, source-consistent computation, and scale matching. Matched pretrained diagnostics, continued pretraining, fused-loss audits, stress endpoints, and MoE training show that these coordinates can behave differently. Low-coefficient settings can preserve validation perplexity (PPL) while stress settings reveal model- and regime-dependent trade-offs among validation quality, update tails, and mixed-precision headroom. The resulting audit protocol therefore selects coefficients and interventions based jointly on transport diagnostics and validation quality instead of prescribing a universally optimal Z-loss variant. Figure~\ref{fig:transport} summarizes the framework.

This framing also identifies a contrast with the usual recipe-level interpretation. In that interpretation, Z-loss is often treated as a scalar auxiliary term whose role is to make $\log Z$ smaller and therefore improve numerical stability. Nevertheless, lowering the scalar log-normalizer tail is not equivalent to lowering the transported update tail. Raw-logit Z-loss can make the reported Z-loss curve look successful while leaving output-to-hidden gain, tied-path coupling, adaptive moment estimation (Adam) state pressure, or static loss-scale headroom essentially unchanged or worse. Conversely, centering, factorization, gain-aware weighting, and router scale matching are not interchangeable ways to reduce one scalar; these methods are interventions on different coordinates of the same backward source.

\begin{figure}[t!]
	\centering
	\resizebox{\linewidth}{!}{%
		\begin{tikzpicture}[
			font=\sffamily\scriptsize,
			>=Latex,
			node/.style={rounded corners=4pt, line width=0.55pt, align=center,
					inner sep=3.2pt, text=TransportInk},
			scalar/.style={node, draw=TransportBlue!58, fill=TransportBlue!4,
					minimum width=4.82cm, minimum height=0.68cm},
			source/.style={node, draw=TransportRed!62, fill=TransportRed!5,
					minimum width=5.34cm, minimum height=0.88cm},
			path/.style={node, minimum width=2.66cm, minimum height=0.98cm,
					text width=2.46cm},
			update/.style={node, draw=TransportInk!60, fill=TransportPanel,
					minimum width=5.74cm, minimum height=0.82cm, text width=5.50cm},
			arrow/.style={-{Latex[length=1.75mm,width=1.22mm]}, line width=0.52pt,
			shorten >=1.8pt, shorten <=1.4pt},
			biarrow/.style={{Latex[length=1.45mm,width=1.02mm]}-{Latex[length=1.45mm,width=1.02mm]},
			line width=0.52pt, shorten >=1.8pt, shorten <=1.4pt},
			rail/.style={line width=0.42pt, draw=TransportInk!18},
			junction/.style={circle, inner sep=0pt, minimum size=1.7pt},
			tag/.style={font=\sffamily\tiny, text=TransportInk!66,
					fill=TransportPanel, inner xsep=3.2pt, inner ysep=1.1pt},
			group/.style={rounded corners=6pt, fill=TransportPanel, draw=TransportInk!12,
					line width=0.5pt}
			]
			\node[scalar] (scalar) at (0,0) {
				Forward scalar\\
				$\mathcal{L}_Z\coloneqq\lambda(\log Z-c)^2$
			};
			\node[source] (source) at (0,-1.25) {
				Z-loss logit gradient, the backward source\\
				$\delta_z^Z=2\lambda(\log Z-c)p$
			};

			\foreach \x/\h/\col in {0.39/0.10/TransportBlue,
					0.60/0.18/TransportTeal,
					0.81/0.32/TransportRed,
					1.02/0.24/TransportAmber,
					1.23/0.14/TransportGreen}
			\draw[line width=0.96pt, draw=\col!55, rounded corners=0.5pt]
			($(source.south west)+(\x,0.13)$) -- ++(0,\h);

			\node[path, draw=TransportTeal!58, fill=TransportTeal!4] (impl) at (-2.32,-3.16) {
				\mbox{Implementation reuse}\\
				\mbox{log-sum-exp stats}\\
				$\Delta_{\rm src},\Delta_\theta$
			};
			\node[path, draw=TransportAmber!62, fill=TransportAmber!5] (shift) at (2.32,-3.16) {
				Common shift\\
				$z=\widetilde z+\mu\mathbf{1}$\\
				remove $2\lambda\mu p$
			};
			\node[path, draw=TransportBlue!58, fill=TransportBlue!4] (gain) at (-2.32,-4.42) {
				Output and tied paths\\
				$\delta_u^Z\coloneqq W_U^\top\delta_z^Z$\\
				tied $E_{\rm out}+E_{\rm in}$
			};
			\node[path, draw=TransportGreen!58, fill=TransportGreen!4] (router) at (2.32,-4.42) {
				Router scale\\
				top-$k$ reductions\\
				$\lambda_R\alpha$ and $N_{\rm tl}\alpha$
			};
			\node[update] (update) at (0,-5.88) {
				Optimizer-facing updates\\
				log-normalizer and gain tails, Adam state, and half-precision headroom
			};

			\coordinate (hubTop) at (0,-2.10);
			\coordinate (railTop) at (0,-2.62);
			\coordinate (railBottom) at (0,-5.18);
			\coordinate (forkImpl) at (0,-3.09);
			\coordinate (forkShift) at (0,-3.23);
			\coordinate (forkGain) at (0,-4.35);
			\coordinate (forkRouter) at (0,-4.49);
			\node[tag] at (0,-2.34) {transport coordinates and diagnostics};

			\draw[arrow, draw=TransportRed!58] (scalar.south) -- (source.north);
			\draw[arrow, draw=TransportRed!54] (source.south) -- (hubTop);
			\draw[rail] (railTop) -- (railBottom);
			\node[junction, fill=TransportInk!26] at (forkImpl) {};
			\node[junction, fill=TransportInk!26] at (forkShift) {};
			\node[junction, fill=TransportInk!26] at (forkGain) {};
			\node[junction, fill=TransportInk!26] at (forkRouter) {};
			\draw[biarrow, draw=TransportTeal!54] (forkImpl) -- ($(impl.east)+(0,0.07)$);
			\draw[biarrow, draw=TransportAmber!56] (forkShift) -- ($(shift.west)+(0,-0.07)$);
			\draw[biarrow, draw=TransportBlue!54] (forkGain) -- ($(gain.east)+(0,0.07)$);
			\draw[biarrow, draw=TransportGreen!54] (forkRouter) -- ($(router.west)+(0,-0.07)$);
			\draw[arrow, draw=TransportInk!48] (railBottom) -- (update.north);

			\begin{scope}[on background layer]
				\node[group, inner sep=5.0pt,
					fit=(scalar) (source) (impl) (shift) (gain) (router) (update)] {};
				\node[group, fill=TransportPanel, inner xsep=5.5pt, inner ysep=8.6pt,
					fit=(impl) (shift) (gain) (router)] {};
			\end{scope}
		\end{tikzpicture}%
	}
	\caption{Backward-transport view of Z-loss. The scalar penalty creates the logit gradient $\delta_z^Z$, termed the backward source. Architecture and implementation determine whether this vector is preserved, rescaled, amplified, or coupled into shared parameters before the transported source reaches optimizer-facing updates.}
	\label{fig:transport}
\end{figure}
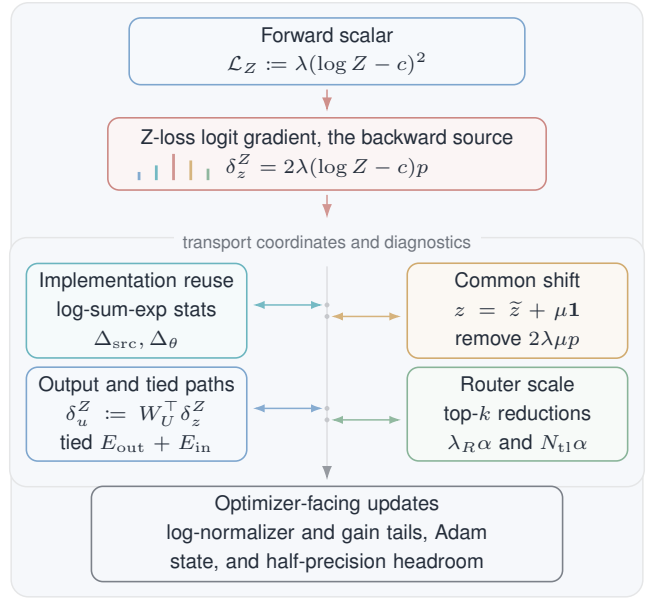

\section{Backward Transport View}

This section develops the backward-transport view by treating the Z-loss logit gradient as the source injected into backpropagation. We first specify the primary language-modeling objective and auxiliary Z-loss, then separate source construction from architectural transport, identify the coordinates and pathways that modify the backward source, and use these results to motivate the diagnostics and experiments that follow. All proofs and more detailed theoretical developments are provided in the Appendix.

\subsection{Language-Modeling Objective and Z-Loss Setup}

Z-loss is added to the standard language-modeling loss. For an autoregressive token sequence $x_{1:T}$, the primary objective is the CE for next-token prediction,
\begin{align*}
	\mathcal L_{\mathrm{CE}}
	=-\frac{1}{T-1}\sum_{t=1}^{T-1}
	\log p_\theta(x_{t+1}\mid x_{\le t}).
\end{align*}
Applied to the deployed-logit log-normalizer, the Z-loss term is added as an auxiliary regularizer: $\mathcal L_{\mathrm{total}}=\mathcal L_{\mathrm{CE}}+\mathcal L_Z$. This setup also delineates the scope of our stability claims: Z-loss directly constrains that log-normalizer rather than output-head precision or representation geometry. Smaller logit excursions may reduce finite-precision roundoff exposure, but any resulting changes in output-embedding geometry, numerical headroom, or optimizer-facing updates are downstream effects that must be evaluated separately.

The penalty studied here has the form
\begin{equation}
	\begin{aligned}
		\mathcal{L}_Z & \coloneqq \lambda(\log Z-c)^2,    &
		Z             & \coloneqq \sum_{i=1}^{V}\exp z_i, &
		\lambda       & \ge 0,
	\end{aligned}
	\label{eq:zloss}
\end{equation}
for a softmax-axis logit vector $z\in\R^V$, log-normalizer target $c$, and coefficient $\lambda$. The target $c$ is the zero-penalty value of $\log Z$, and $V$ denotes the size of the relevant softmax axis. Prior final-softmax and router formulations commonly use $\lambda(\log Z)^2$, corresponding to $c=0$ \citep{shazeer2018mesh,chowdhery2023palm,zoph2022stmoe}. We instead study the target-shifted generalization in \Eqref{eq:zloss} and use $c=\log V$ unless stated otherwise; $c$ is fixed rather than tuned, while coefficient sweeps vary $\lambda$. This choice yields zero violation for the zero-logit uniform vector. For raw logits, setting $c=\log V$ selects a common-shift target, whereas for centered logits, the same setting defines the uniform reference for the remaining relative-logit coordinate. Accordingly, the Standard Z-loss rows in our tables apply \Eqref{eq:zloss} to raw logits with $c=\log V$; conventional formulations commonly use $c=0$.

This scalar definition of Z-loss in \Eqref{eq:zloss} specifies the forward loss value. However, the scalar definition alone cannot determine the resulting gradient-based parameter update. The update also depends on how the Z-loss logit gradient, which we term the backward source in the next subsection, is transported through the architecture and training implementation. A scalar audit can therefore certify the quantity being added to the objective while missing the vector actually injected into training. This distinction matters precisely in the finite-precision and training-divergence settings where Z-loss is used. The reported penalty can remain unchanged across output-head gauges, tied embeddings, fused low-precision losses, adaptive optimizers, and top-$k$ router reductions, even as these factors change where and how strongly the logit gradient reaches parameters. A complete analysis must therefore identify this gradient and the transport path through which the gradient acts, in addition to measuring the Z-loss value.

\subsection{Z-Loss Logit Gradient and Transport}

Let $z\in\R^V$ be the output-logit vector, $p\coloneqq\softmax(z)$, and $Z\coloneqq\sum_i e^{z_i}$. Differentiating \Eqref{eq:zloss} gives
\begin{equation}
	\delta_z^Z
	\coloneqq
	\frac{\partial \mathcal{L}_Z}{\partial z}
	=
	2\lambda(\log Z-c)p .
	\label{eq:source}
\end{equation}
Throughout, source is shorthand for this Z-loss logit gradient, equivalently the logit-space adjoint $\delta_z^Z$. The logit-space adjoint is a covector at the logit boundary. The component $\delta_{z,i}^Z\coloneqq\partial\mathcal L_Z/\partial z_i$ represents the local sensitivity to logit $z_i$. The term source reflects the adjoint's role as the starting backward signal that the Z-loss branch of reverse-mode differentiation transports toward hidden states and parameters. The adjoint is neither a data source nor a parameter gradient. In a dense output head, the source has one entry per vocabulary item for each token; in an MoE router, the source has one entry per expert for each token-layer routing decision. Unless explicitly qualified as a total CE+Z source, source refers only to the auxiliary Z-loss contribution; the CE logit gradient is the separate vector $p-y$.

The source has two separable factors. The first is a signed scalar coefficient $2\lambda(\log Z-c)$, whose absolute value $2\lambda|\log Z-c|$ we call the scalar source amplitude. The second is the softmax-shaped vector $p$, which we call the source shape. Consequently, the full Euclidean source magnitude is $\norm{\delta_z^Z}_2=2\lambda|\log Z-c|\norm{p}_2$. For fixed $\lambda>0$, this identity can also be written as
\begin{align*}
	\norm{\delta_z^Z}_2
	=2\sqrt{\lambda\mathcal L_Z}\norm{p}_2,
\end{align*}
so the exact per-instance source vanishes with the scalar penalty. The scalar alone nevertheless leaves the source sign and softmax shape unspecified and does not determine amplification through the logit Jacobian. When discussing gain, we view source shape through both the concentration $\norm{p}_2$ and the normalized direction $p/\norm{p}_2$. The downstream architecture then determines how this shaped source is transported. For any parameter block $\theta$ with logit Jacobian
\begin{align*}
	J_\theta \coloneqq \frac{\partial z}{\partial \vecop(\theta)},
\end{align*}
the Z-loss gradient is
\begin{equation}
	\vecop(\nabla_\theta\mathcal{L}_Z)
	=
	J_\theta^\top \delta_z^Z .
	\label{eq:transport}
\end{equation}
Thus, the backward terminology has a direct operational meaning. The complete causal pipeline can be written compactly as
\begin{equation}
	\begin{aligned}
		\widehat\delta_z^Z
		 & =2\lambda\alpha(\widehat{\log Z}-c)\widehat p,                         \\
		g_\theta^Z
		 & =J_\theta^\top\widehat\delta_z^Z,                                      \\
		(\theta^+,s^+)
		 & =\operatorname{Optimizer}(\theta,s,g_\theta^{\mathrm{CE}}+g_\theta^Z),
	\end{aligned}
	\label{eq:causal_pipeline}
\end{equation}
where hats denote the statistics actually supplied by a fused or low-precision implementation, $\alpha$ is the per-instance reduction weight, and $s$ is optimizer state. Common shift and $c$ affect source amplitude; implementation precision affects the hatted source statistics; top-$k$ reduction affects $\alpha$; tied pathways and output-to-hidden gain enter through $J_\theta^\top$; and the optimizer transforms the total CE+Z gradient in the final line. Source construction, architectural transport, and optimizer transformation are distinct stages rather than interchangeable transport paths. For a single exact softmax instance, $\alpha=1$, $\widehat{\log Z}=\log Z$, and $\widehat p=p$.

In plain gradient descent with step size $\eta$, the Z-loss contribution to the parameter step is
\begin{align*}
	\Delta_Z\vecop(\theta)
	\coloneqq
	-\eta J_\theta^\top \delta_z^Z .
\end{align*}
This transport distinction has three direct consequences. First, implementations that agree on the forward scalar can transport different backward sources and therefore apply different parameter updates. Second, router reductions that use the same nominal $\lambda_R$ can assign different absolute coefficients to each routing decision and different scales relative to the token-layer mean. Finally, when the output head and input embeddings are tied, an auxiliary loss applied at the output logits is transported through both output and input pathways, turning a seemingly output-side regularizer into a coupled update to shared lexical memory.

For an adaptive optimizer, an isolated additive Z-loss step need not be well-defined because the optimizer transforms the total gradient together with the accumulated optimizer state; this optimizer coupling is why \Eqref{eq:causal_pipeline} places $g_\theta^{\mathrm{CE}}+g_\theta^Z$ at the optimizer boundary. Taking norms in \Eqref{eq:transport} gives
\begin{equation}
	\norm{\nabla_\theta\mathcal{L}_Z}_F
	\le
	2\lambda|\log Z-c|
	\norm{J_\theta}_{\mathrm{op}}
	\norm{p}_2 .
	\label{eq:transport_bound}
\end{equation}
\Eqref{eq:transport_bound} is the central diagnostic lens. The magnitude of the transported Z-loss gradient is not determined by the forward scalar alone; the transported-gradient magnitude also depends on source shape and the relevant Jacobian gain.

In matched low-coefficient regimes, the Z-loss gradient can remain small relative to the CE gradient. For a fixed coefficient and a single exact softmax instance, a vanishing scalar Z-loss also implies a vanishing logit-space source. The salient issue is the relationship between scalar aggregates and backward transport: a scalar Z-loss curve may improve while rare instances and architectural or implementation factors produce large pre-clip gradient tails, change Adam state, reduce mixed-precision headroom, or change router update scale. These effects cannot be diagnosed from validation loss and scalar Z-loss alone, but the effects determine what the optimizer actually receives. We therefore treat Z-loss as a quantity to audit and debug at source, transport, and update levels rather than only as a forward regularizer.

A reported Z-loss value specifies the forward penalty, but the transported vector $J_\theta^\top\delta_z^Z$ is what enters the optimizer. Comparable reports therefore need source-shape, transport-path, and optimizer-facing diagnostics rather than scalar curves alone.

The finite-precision rationale is two-sided: log-normalizer control can preserve forward representability, while the resulting backward source can consume numerical headroom after architectural transport.

\subsection{Common-Shift Sensitivity}

We first isolate how a token-wise common logit shift enters both the scalar penalty and the injected logit source. For a one-hot next-token target $y$, let $\ell_{\mathrm{CE}}$ denote the per-token CE loss, in contrast to the sequence-averaged objective $\mathcal L_{\mathrm{CE}}$ above:
\begin{align*}
	\ell_{\mathrm{CE}}(z,y)
	\coloneqq
	-\sum_{i=1}^{V}y_i\log\softmax(z)_i .
\end{align*}
If $z'=z+a\mathbf 1$, then $\softmax(z')=\softmax(z)=p$, and CE is invariant both in value and along the common-shift direction:
\begin{equation}
	\frac{\partial}{\partial a}
	\ell_{\mathrm{CE}}(z+a\mathbf 1,y)
	=\mathbf 1^\top(p-y)=0.
	\label{eq:ce_shift_direction}
\end{equation}
This contrast leads to the following common-shift property of Z-loss.
\begin{proposition}[Z-loss common-shift sensitivity]
	For any fixed $a\in\R$ independent of $z$,
	\begin{align*}
		\mathcal{L}_Z(z+a\mathbf{1})
		=
		\lambda(\log Z(z)+a-c)^2,
	\end{align*}
	and
	\begin{align*}
		\nabla_z\left[\mathcal{L}_Z(z+a\mathbf{1})\right]
		=
		2\lambda(\log Z(z)+a-c)\softmax(z).
	\end{align*}
	Equivalently, the gradient with respect to the shifted logit variable $x$ at $x\coloneqq z+a\mathbf{1}$ has the same displayed value.
\end{proposition}

By contrast with \Eqref{eq:ce_shift_direction}, summing the components of the source in Proposition~1 gives
\begin{equation}
	\frac{\partial}{\partial a}
	\mathcal L_Z(z+a\mathbf 1)
	=\mathbf 1^\top\delta_{z+a\mathbf 1}^Z
	=2\lambda(\log Z(z)+a-c).
	\label{eq:z_shift_direction}
\end{equation}
The common-shift coordinate is therefore forward-invisible and backward-inactive for CE, but backward-active once Z-loss is added. Thus, monitoring CE alone cannot reveal common-shift drift, whereas Z-loss changes directly with that drift. Although the prediction is unchanged, the altered source is transported to a dense hidden state as $\delta_u^Z=W_U^\top\delta_z^Z$ and can therefore change parameter gradients and optimizer state. Centering removes this forward-invisible but Z-loss-active coordinate before source construction.

\subsection{Centered Output Heads}

Let $W_U\in\R^{V\times d}$ be a bias-free output projection, also called the output head or unembedding matrix, $u\in\R^d$ the final hidden state, and $\bar w\coloneqq V^{-1}\mathbf{1}^\top W_U$. Define
\begin{equation}
	\begin{aligned}
		\widetilde W_U & \coloneqq W_U - \mathbf{1}\bar w,                      \\
		\widetilde z   & \coloneqq \widetilde W_Uu = W_Uu-(\bar w u)\mathbf{1}.
	\end{aligned}
	\label{eq:centered_head}
\end{equation}

The following proposition formalizes the standard shift invariance of softmax for output-head centering \citep{blanchard2021accurate}.

\begin{proposition}[Centering preserves CE]
	For any fixed hidden state $u$ and a bias-free output head, both the token softmax distribution and token CE remain unchanged in real arithmetic when the deployed logits $z=W_Uu$ are replaced by $\widetilde z=\widetilde W_Uu=z-\alpha\mathbf{1}$, where $\alpha\coloneqq\bar w u$.
\end{proposition}

Centering changes the Z-loss path because centering removes the common-shift channel $\mu\coloneqq\bar w u$ from the deployed-logit source coordinate. More generally, write
\begin{align*}
	z=\widetilde z+\mu\mathbf{1},
	\qquad
	\mathbf{1}^\top \widetilde z=0.
\end{align*}
Then
\begin{equation}
	\log Z = \mu+\log \widetilde Z,
	\qquad
	\widetilde Z \coloneqq \sum_i e^{\widetilde z_i}.
\end{equation}

Output-head centering preserves CE but intentionally changes the Z-loss source coordinate. Under the shared-coefficient, shared-target convention, output-head centering removes the deployed source component $2\lambda\mu p$. The effect of centering should be judged by transport diagnostics such as $P_Z^{99.9}$ and $A_p^{99}$, both defined later in the subsection on diagnostic metrics, because removing the common-shift channel can be beneficial or introduce different trade-offs across model geometries.

\subsection{Tied-Embedding Coupling}

With tied embeddings $W_U=E$ and logits $z=Eu$, the Z-loss gradient with respect to $E$ has two conceptually separate paths.

\begin{proposition}[Tied embedding gradient decomposition]
	Consider untied copies $E_{\mathrm{in}}$ and $E_{\mathrm{out}}$ and impose the tying constraint $E_{\mathrm{in}}=E_{\mathrm{out}}=E$ after differentiation. Then
	\begin{equation}
		\nabla_E\mathcal{L}_Z
		=
		\nabla_{E_{\mathrm{out}}}\mathcal{L}_Z
		+
		\nabla_{E_{\mathrm{in}}}\mathcal{L}_Z .
		\label{eq:tied_decomp}
	\end{equation}
\end{proposition}

Under this decomposition, the output-side term is the local head update induced by the logit adjoint. The input-side partial derivative follows the transported path. The source $\delta_z^Z$ first induces the hidden-state adjoint $E^\top\delta_z^Z$, which is then backpropagated through the Transformer body to the input embedding occurrences. The input-side term is transported through
\begin{align*}
	\delta_z^Z \rightarrow E^\top\delta_z^Z
	\rightarrow \text{Transformer body}
	\rightarrow \nabla_{E_{\mathrm{in}}}\mathcal{L}_Z .
\end{align*}
Thus, an output-side log-normalizer regularizer becomes a multi-path update to a shared lexical table.

\subsection{Output-to-Hidden Gain}

The hidden-state injection from Z-loss is
\begin{equation}
	\delta_u^Z
	\coloneqq
	W_U^\top \delta_z^Z
	=
	2\lambda(\log Z-c)W_U^\top p,
	\label{eq:hidden_injection}
\end{equation}
and hence
\begin{equation}
	\norm{\delta_u^Z}_2
	\le
	2\lambda|\log Z-c|\norm{W_U}_{\mathrm{op}}\norm{p}_2.
	\label{eq:hidden_bound}
\end{equation}
We measure the token-level shape-only local gain
\begin{equation}
	a_t\coloneqq\frac{\norm{W_U^\top p_t}_2}{\norm{p_t}_2+\varepsilon}.
	\label{eq:ap}
\end{equation}
We use $\bar A_p\coloneqq\E_t a_t$ for the mean gain and $A_p^{99}\coloneqq Q_{0.99,t}(a_t)$ for the 99th-percentile (p99) gain. Throughout the paper, unembedding, output head, and output projection refer to the same output-side matrix $W_U$ unless a tied input-embedding path is explicitly being discussed.

\begin{proposition}[Gain is anisotropy-weighted]
	Let $W_U=U\Sigma V^\top$, $q\coloneqq p/\norm{p}_2$, and $a(p)\coloneqq\norm{W_U^\top p}_2/\norm{p}_2$, the $\varepsilon$-free counterpart of \Eqref{eq:ap}. Then
	\begin{equation}
		a(p)^2
		=
		\sum_j \sigma_j^2 \inner{q}{u_j}^2
		\le
		\sigma_{\max}(W_U)^2 .
		\label{eq:ap_svd}
	\end{equation}
\end{proposition}

The expansion also shows the alignment claim. For fixed $\norm{p}_2$, the contribution from direction $u_j$ is weighted by $\sigma_j^2$. Among feasible source directions, the hidden injection is therefore largest when source mass lies in left singular directions associated with high singular values. As a linear algebra bound over unit vectors, the upper bound is saturated exactly when $q$ lies entirely in the left singular subspace associated with $\sigma_{\max}(W_U)$. For softmax-derived $q$, this equality condition may be infeasible. This alignment property explains the operational sense in which unembedding geometry, not only $\log Z$, predicts transported update size.

Raw-logit Z-loss can reduce $\log Z$ tails while leaving $W_U^\top p$ and rare optimizer-facing update tails exposed. Diagnostics such as $A_p^{99}$ and gradient-event endpoints are needed to distinguish log-normalizer-tail control from transported-update-tail control.

\subsection{Router Effective Scale}

An MoE router produces, for each token-layer decision, a vector of expert scores $r_{t,\ell}\in\R^{N_{\mathrm{exp}}}$ and applies softmax or top-$k$ selection over the expert axis. Router Z-loss is therefore the same log-normalizer penalty as \Eqref{eq:zloss}, instantiated on the expert-score axis rather than the vocabulary axis. For decision $(t,\ell)$, the router source is the expert-score gradient $\delta_{r,t,\ell}^{R}\coloneqq\partial\mathcal L_R/\partial r_{t,\ell}\in\R^{N_{\mathrm{exp}}}$, with one component per expert. The router source has the same softmax-shaped form as the dense-head source, while top-$k$ routing and reduction denominators determine the effective scale transported through the router. Let $N_{\mathrm{tl}}$ be the number of token-layer router decisions and $\alpha_{t,\ell}$ the reduction weight for decision $(t,\ell)$. Relative to the token-layer mean, the effective scale is $N_{\mathrm{tl}}\alpha_{t,\ell}$. Under the conventions used in the experiments, active-route mean has a scale of $1/k$, while active-route sum has a scale of $k$.

The same $\lambda_R$ can imply different per-decision gradients once top-$k$ and reduction denominators change. Router reports should therefore include the reduction convention, the absolute per-decision coefficient, and the scale relative to the token-layer mean.

\section{Experiments}

\subsection{Setup}

This section empirically evaluates the preceding analysis and diagnostic framework in representative modern Transformer settings where Z-loss is used.

We evaluate raw and centered heads, standard and centered Z-loss, factorized and gain-aware objectives, and router-scale variants using text from WikiText-103 and FineWeb-Edu. The comparison spans pretrained-model diagnostics, matched continued pretraining, implementation audits, stress endpoints, and controlled MoE training. Unless otherwise stated, the FineWeb-Edu experiments reported here use streamed blocks from the FineWeb-Edu \texttt{sample-10BT} training split. Dense runs on models in the GPT-2 family use matched replications, stated token budgets, 32-bit floating-point (fp32) master weights with bfloat16 (bf16) autocast, and log-normalizer, gradient-tail, optimizer-state, and mixed-precision endpoints; the implementation and overflow audits also evaluate 16-bit floating-point (fp16) storage. The Appendix provides supporting derivations, diagnostic and method conventions, further experimental details, and additional audits and results.

The experiments follow the diagnostic hierarchy rather than rank the methods as universal alternatives. A large p99 $|\mu|$ motivates a centered-deployment audit; a large residual $P_Z^{99.9}$ after centering motivates centered Z-loss; a large $A_p^{99}$ motivates gain and spectral audits; forward--backward mismatch motivates shared-statistics implementation checks; and router-scale mismatch motivates coefficient matching. The Centered head row is a gauge-only control with $\lambda_{\rm aux}=0$. In exact arithmetic, the Centered head row's CE value and CE gradient equal those of the raw CE baseline, so the row isolates the deployed coordinate rather than adding an auxiliary update. Because the diagnostic Z-loss and $P_Z^{99.9}$ are evaluated in each method's deployed-logit coordinate, differences between raw and centered coordinates are not coordinate-free performance gains; in pretrained audits the differences are deterministic consequences to be judged together with $A_p^{99}$ and PPL. Validation comparisons therefore report exact matched mean PPL differences within each run regime, making quality changes explicit alongside transport effects.

\subsection{Pretrained Output-Head Geometry}

This comparison tests the common-shift diagnosis and the geometry-dependent consequences of centering. Table~\ref{tab:pretrained} shows that row-centered output-head deployment preserves PPL while sharply reducing common-shift and Z-loss transport diagnostics for several models. Raw heads in the GPT-2 family have very large common-shift tails. Centered heads keep CE unchanged while reducing $P_Z^{99.9}$ by factors of $4.9$, $12.8$, and $12.7$ for DistilGPT-2, GPT-2, and GPT-2 Medium, respectively. The larger Pythia-1B diagnostic provides an important contrast. The Pythia-1B raw common-shift tail is already small, and centering increases $A_p^{99}$. This increase is predicted by the coordinate analysis rather than treated as an anomalous exception. In the scalar coordinate, a raw $\mu$ can partially cancel $\log\widetilde Z-c$, so removing the raw common-shift term can increase the centered violation. In the gain coordinate, $\widetilde W_U^\top p=W_U^\top p-\bar w^\top$, and subtracting the fixed output-row mean need not reduce the norm when the mean's alignment with $W_U^\top p$ is unfavorable. The contrast therefore reinforces treating centering as a gauge intervention whose effect is evaluated with transport diagnostics, not as a uniformly beneficial normalization.

Figure~\ref{fig:gradient_geometry}(a) makes the common-shift geometry visible by showing that centering translates every GPT-2 hidden-source shape by the same output-row mean while leaving the softmax distribution unchanged. Figure~\ref{fig:gradient_geometry}(b) shows that the transported gain is concentrated in a low-rank, anisotropic subspace rather than determined by scalar source amplitude alone. Figure~\ref{fig:gradient_geometry}(c) directly realizes the identity $\norm{\delta_u^Z}_2=A_p\norm{\delta_z^Z}_2$ token by token. The spread across fixed-gain rays is the empirical geometry that a scalar Z-loss curve omits.

\begin{figure*}[t!]
	\centering
	\includegraphics[width=\textwidth]{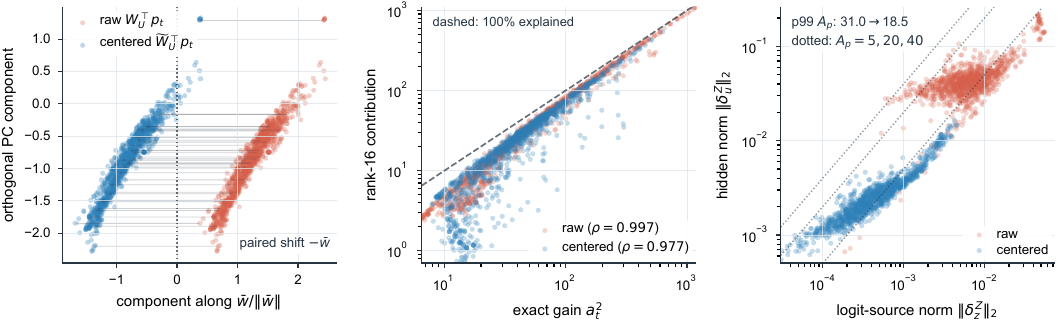}
	\caption{Empirical Z-loss gradient geometry on fp32 GPT-2 predictions from WikiText-103. The left view shows that centering translates every hidden-source shape by the fixed vector $-\bar w^\top$ while preserving $p_t$. The middle view shows that singular contributions closely track exact $a_t^2$. The raw and centered correlations are $0.997$ and $0.977$, respectively. The right view shows that $\norm{\delta_u^Z}_2=A_p\norm{\delta_z^Z}_2$ exposes token-dependent amplification, and centering lowers p99 $A_p$ from 31.0 to 18.5. Here, $\lambda_{\rm diag}=10^{-4}$ and $c=\log 50257$.}
	\label{fig:gradient_geometry}
\end{figure*}

\begin{table}[t!]
	\centering
	\scriptsize
	\resizebox{\columnwidth}{!}{%
		\begin{tabular}{@{}lcrrrrr@{}}
			\toprule
			Model                         & Head     & PPL   & Diag. Z  & $P_Z^{99.9}$ & $A_p^{99}$ & p99 $|\mu|$         \\
			\midrule
			\multirow{2}{*}{DistilGPT-2}  & raw      & 65.01 & 0.394    & 110.11       & 32.02      & 99.94               \\
			                              & centered & 65.01 & 0.00517  & 22.48        & 19.52      & $6.0\times 10^{-6}$ \\
			\addlinespace
			\multirow{2}{*}{GPT-2}        & raw      & 42.58 & 1.211    & 263.45       & 33.51      & 263.91              \\
			                              & centered & 42.58 & 0.00523  & 20.61        & 19.16      & $8.0\times 10^{-6}$ \\
			\addlinespace
			\multirow{2}{*}{GPT-2 Medium} & raw      & 29.56 & 1.040    & 271.56       & 30.60      & 273.91              \\
			                              & centered & 29.56 & 0.00580  & 21.45        & 14.27      & $3.5\times 10^{-6}$ \\
			\addlinespace
			\multirow{2}{*}{GPT-2 Large}  & raw      & 25.43 & 0.000692 & 12.82        & 8.42       & 8.47                \\
			                              & centered & 25.43 & 0.00573  & 18.27        & 8.90       & $7.8\times 10^{-7}$ \\
			\midrule
			\multirow{2}{*}{Pythia-160M}  & raw      & 41.68 & 59.33    & 785.73       & 61.35      & 756.65              \\
			                              & centered & 41.68 & 0.0733   & 39.67        & 6.42       & $4.6\times 10^{-3}$ \\
			\addlinespace
			\multirow{2}{*}{Pythia-410M}  & raw      & 24.81 & 0.0145   & 19.85        & 3.09       & 6.80                \\
			                              & centered & 24.81 & 0.00609  & 17.46        & 4.61       & $2.0\times 10^{-5}$ \\
			\addlinespace
			\multirow{2}{*}{Pythia-1B}    & raw      & 17.42 & 0.00584  & 17.32        & 2.73       & 2.74                \\
			                              & centered & 17.42 & 0.00646  & 17.35        & 4.39       & $6.6\times 10^{-7}$ \\
			\bottomrule
		\end{tabular}
	}%
	\caption{Pretrained-model diagnostics on WikiText-103 validation. Rows labeled centered use the row-centered deployed-logit convention defined above. Tail metrics are reported on the deployed logits for each row; centering preserves CE but fixes the common-shift gauge, so centering changes the Z-loss transport coordinate. Lower PPL indicates better validation quality, while lower diagnostic Z-loss, abbreviated as Diag. Z, and lower $P_Z^{99.9}$, $A_p^{99}$, and $|\mu|$ indicate smaller penalty or transport tails. Gradient ratios and cosines characterize auxiliary-gradient strength and CE alignment; the gradient ratios and cosines have no universally preferred direction unless stated otherwise.}
	\label{tab:pretrained}
\end{table}

To test whether the WikiText comparison of raw and centered output heads is corpus-specific, Appendix Table~\ref{tab:fineweb_pretrained} repeats the pretrained evaluation on streamed FineWeb-Edu blocks. GPT-2 and GPT-2 Medium reproduce the large common-shift pattern, whereas GPT-2 Large and Pythia-410M remain in a geometry-dependent regime.

Additional tied-embedding pathway measurements, including the untied Pythia-160M contrast, confirm that tying changes the transport path, while output-head common shift remains a separate source factor whose effect depends on model geometry.

Appendix Table~\ref{tab:spectral_audit} reports a low-rank unembedding spectral audit: dominant singular directions explain much of the token-level gain in highly anisotropic heads, while the centered-head results show that reducing $\sigma_1$ does not by itself guarantee a smaller $A_p^{99}$.

\subsection{Source-Consistency and Implementation Audits}

These audits test whether forward scalar agreement predicts backward-source and optimizer-facing agreement. Appendix Tables~\ref{tab:source_audit}, \ref{tab:triton_fused}, and~\ref{tab:update_audit} report source-space, \mbox{Triton}-based fused-kernel, and optimizer-facing update audits that support the implementation claim used throughout the paper.\footnote{\mbox{Triton} is the implementation framework for a controlled CE+Z-loss GPU kernel.} Together, these audits test whether fused forward and backward computations share the softmax statistics that define the source. Matching forward Z-loss scalar values is not sufficient, because the backward source and the transported optimizer-facing update can differ on corpus-derived model logits even when the forward scalar error is small. The source-space, fused-kernel, and optimizer-facing update audits distinguish two failure modes. Reusing the forward log-sum-exp statistics removes inconsistency caused by backward reconstruction, but a source-consistent formula cannot recover precision already lost through quantized logit storage.

\subsection{Dense Continued-Pretraining Interventions}

This experiment contrasts scalar-tail control with output-to-hidden gain control. Table~\ref{tab:finetune} reports the matched-run GPT-2 continued pretraining experiment on WikiText-103. The CE-only setting is a strong baseline. After 4800 optimizer steps, the CE-only setting reaches a PPL of 21.046$\pm$0.012. The matched mean PPL differences relative to CE are $-0.011$ for Centered head, $+0.026$ for Standard Z-loss, and $-0.007$ for Centered Z-loss. We report these differences directly rather than infer formal equivalence from three-run standard deviations. The larger and more seed-consistent effect is the geometry change. Centered head deployment lowers $A_p^{99}$ from 26.4 to 17.0 and removes the common-shift coordinate; because centered head deployment has no auxiliary loss, the Centered head row isolates the coordinate change. Standard Z-loss reduces $P_Z^{99.9}$ from 246.8 to 21.8 but leaves the $A_p^{99}$ gain essentially unchanged. Centered Z-loss combines residual log-normalizer control with the centered coordinate, reducing $P_Z^{99.9}$ to 20.3 and $A_p^{99}$ to 16.8 at the PPL difference.

\begin{table}[t!]
	\centering
	\scriptsize
	\resizebox{\columnwidth}{!}{%
		\begin{tabular}{@{}lrrrrrr@{}}
			\toprule
			Method          & $\lambda_{\rm aux}$ & PPL              & $\Delta$PPL & Diag. Z & $P_Z^{99.9}$ & $A_p^{99}$ \\
			\midrule
			CE baseline     & 0                   & $21.046\pm0.012$ & 0           & 0.1174  & 246.8        & 26.37      \\
			Centered head   & 0                   & $21.035\pm0.013$ & $-0.011$    & 0.00879 & 32.8         & 16.98      \\
			\midrule
			Standard Z-loss & $1\times 10^{-4}$   & $21.072\pm0.013$ & $+0.026$    & 0.00245 & 21.8         & 26.30      \\
			Centered Z-loss & $1\times 10^{-4}$   & $21.039\pm0.013$ & $-0.007$    & 0.00609 & 20.3         & 16.84      \\
			\bottomrule
		\end{tabular}
	}%
	\caption{GPT-2 continued pretraining on WikiText-103. All methods start from the same pretrained checkpoint. $\lambda_{\rm aux}$ is the applied auxiliary coefficient, Diag. Z is evaluated on the deployed logits for each method, and $\Delta$PPL is the matched mean difference from CE.}
	\label{tab:finetune}
\end{table}

Appendix Table~\ref{tab:fineweb_ft} reports a disjoint-block GPT-2 replication on FineWeb-Edu, and Appendix Table~\ref{tab:fineweb_medium_ft} extends the comparison to GPT-2 Medium with fresh tokens. Both reproduce the same separation. Centered variants reduce $A_p^{99}$, whereas standard raw-logit Z-loss does not.

Appendix Table~\ref{tab:gpt2_diag} reports a separate 1200-step GPT-2 diagnostic with post-training gradient instrumentation. The diagnostic confirms that the tied-table auxiliary gradient is nonzero and path-coupled and that the alignment of the gradient with CE should be measured rather than assumed.

\subsection{WikiText-103 High-Coefficient Stress Tests}

With the coefficient increased tenfold to $\lambda_{\rm aux}=1\times10^{-3}$, Table~\ref{tab:stress_diag} tests optimizer-facing consequences. Standard and gain-aware Z-loss reduce log-normalizer tails but leave $A_p^{99}$ near CE. Their p99 gradient norms are 47.81 and 30.25, respectively, compared with 5.58 for CE. Centered Z-loss lowers both $P_Z^{99.9}$ and $A_p^{99}$, keeps the p99 norm at 5.69, and improves mean PPL. The roughly $7\%$ clip rates of the raw-logit variants reflect mostly sub-threshold steps punctuated by rare large events rather than safer tails.

\begin{table}[t!]
	\centering
	\scriptsize
	\resizebox{\columnwidth}{!}{%
		\begin{tabular}{@{}lrrrrrr@{}}
			\toprule
			Method            & PPL               & $P_Z^{99.9}$ & $A_p^{99}$ & \shortstack{p99 pre-clip                  \\$\|g\|_2$} & Clip rate & $R_{\rm aux}$ \\
			\midrule
			CE baseline       & $21.899\pm0.013$  & 257.9        & 26.49      & 5.58                     & 1.00  & 0      \\
			Centered head     & $22.048\pm0.022$  & 34.7         & 16.59      & 5.73                     & 1.00  & 0      \\
			\midrule
			Standard Z-loss   & $21.908\pm0.011$  & 16.9         & 26.16      & 47.81                    & 0.071 & 0.0145 \\
			Centered Z-loss   & $21.840\pm0.015$  & 16.8         & 16.47      & 5.69                     & 1.00  & 0.0202 \\
			Gain-aware Z-loss & $21.890\pm0.0066$ & 19.0         & 26.15      & 30.25                    & 0.073 & 0.0116 \\
			\bottomrule
		\end{tabular}
	}%
	\caption{High-coefficient GPT-2 gradient-tail stress sweep on WikiText-103 with $\lambda_{\rm aux}=1\times 10^{-3}$ for auxiliary-method rows. CE and centered head rows have no applied auxiliary loss, and Diag. Z uses the same coefficient scale. Clip rate is the fraction of optimizer steps whose pre-clip full-model gradient norm exceeds the clipping threshold of 1.0.}
	\label{tab:stress_diag}
\end{table}

\subsection{Router Reduction and MoE Training}

These experiments test the predicted reduction scale and whether coefficient matching restores token-layer-mean behavior. The router-projection audit in Appendix Table~\ref{tab:router} compares token-layer mean, active-route mean, and active-route sum reductions under a fixed nominal router coefficient.

End-to-end MoE results in Appendix Table~\ref{tab:moe} provide an implementation-level equivalence check of the same scale accounting. Matched active-route variants recover token-layer-mean behavior. Because coefficient matching makes the effective objectives equivalent under matched initialization, this recovery confirms that the predicted reduction-scale accounting carries through to end-to-end MoE training.

\section{Conclusion}

Z-loss should be understood not only as a scalar log-normalizer penalty but as a logit-gradient source transported by architecture and implementation. Our backward-transport formulation explains why forward-equivalent losses can produce different updates and why reducing raw-logit tails need not reduce optimizer-facing tails. Across dense and MoE training, implementation audits, and optimizer stress tests, we observed that the common-shift gauge, unembedding anisotropy, tied pathways, and router reductions affect distinct transport coordinates. Intervention choice should therefore follow the relevant diagnostic, with validation quality and transport metrics reported together.

\bibliography{aaai2027}

\clearpage
\appendix

\section{Appendix}

\subsection{Related Work}
\label{app:related_work}

The closest uses of the log-normalizer penalty are final-softmax Z-loss in Mesh TensorFlow and PaLM and router Z-loss in Stable and Transferable Mixture-of-Experts (ST-MoE) \citep{shazeer2018mesh,chowdhery2023palm,zoph2022stmoe}. To avoid a naming ambiguity, the squared log-normalizer penalty in \Eqref{eq:zloss} is mathematically distinct from an earlier loss with the same name. \citet{debrebisson2016zloss} proposed a shift- and scale-invariant spherical surrogate based on the standardized target-class score as an alternative to log-softmax, rather than an auxiliary penalty on $\log Z$. Relative to output-centering and logit-geometry work, our focus is the architecture- and implementation-dependent transport of the resulting adjoint. We develop the broader context on routers, output embeddings, common shifts, logit geometry, systems, and the positioning of this study below.

\paragraph{Z-loss, router-logit control, and MoE reporting.} Conditional computation and sparse experts date back to adaptive mixtures of local experts \citep{jacobs1991adaptive} and the sparsely-gated MoE layer \citep{shazeer2017moe}. GShard, Switch Transformer, and ST-MoE then established large sparse Transformer routing conventions and training-loss-divergence issues \citep{lepikhin2020gshard,fedus2022switch,zoph2022stmoe}. Subsequent sparse language models, routing methods, and systems have continued to emphasize the same design variables. Examples include BASE Layers, Expert Choice routing, V-MoE, dropless MoE systems, and routed scaling-law studies \citep{lewis2021base,zhou2022expertchoice,gale2023megablocks,riquelme2021scaling,clark2022unified}. This line of work shows that top-$k$ routing, balancing, capacity, dynamic expert workload, and compute-normalized scaling remain central design variables. The log-normalizer penalty we study follows the Mesh TensorFlow and PaLM final-softmax convention and the ST-MoE router adaptation \citep{shazeer2018mesh,chowdhery2023palm,zoph2022stmoe}. Our MoE contribution extends this reporting line. Nominal $\lambda_R$ is not sufficient to compare experiments unless the token-layer denominator, top-$k$, active-route convention, capacity policy, absolute per-decision coefficient, and scale relative to the token-layer mean are reported. Recent auxiliary-loss-free load balancing further shows that router-logit control and load balancing can be implemented without directly adding auxiliary gradients to the language-model objective \citep{wang2024auxiliary}. \citet{olmo2025} document modern dense large language model (LLM) training recipes that include logit-growth-control terms and transparent recipe reporting. Our reporting recommendations turn that practice into Z-loss-specific backward-source and transport diagnostics.

\paragraph{Numerical losses and fused implementations.} Implementation audits in this work build on several lines of research in numerical analysis and systems. The relevant numerical and systems literature includes mixed-precision training, floating-point rounding behavior, large-scale model-parallel training stacks, accurate log-sum-exp and softmax computation, online softmax normalization, input- and output-aware fused kernels, programmable GPU-kernel languages and compilers such as \mbox{Triton}, adaptive large-vocabulary softmax, and memory-efficient large-vocabulary CE \citep{micikevicius2018mixed,goldberg1991floating,shoeybi2019megatron,blanchard2021accurate,milakov2018online,dao2022flashattention,tillet2019triton,grave2017efficient,wijmans2024cut}. The PyTorch CE audit is close to the fp32 reference in our tests, consistent with the stable softmax and CE formulations used by mature frameworks \citep{paszke2019pytorch}. Adding Z-loss creates an additional implementation obligation. When the auxiliary path is fused, quantized, or reconstructed separately, the backward source of the auxiliary path must remain consistent with the forward objective. Forward scalar agreement alone does not establish this consistency. Because the Z-loss source is later transported through tied heads, unembedding geometry, and adaptive optimizer normalization, a fused CE+Z-loss path should be validated with source consistency, selected-parameter gradients, and optimizer directions on target-model logits from the relevant corpus distribution.

\paragraph{Output embeddings, weight tying, and common-shift gauges.} Weight tying and output embeddings affect language-model parameterization and generalization \citep{press2017using,inan2017tying}. Work on embedding degeneration and representation geometry further shows that lexical vectors can concentrate in narrow cones or dominant mean and top directions. This pattern appears in maximum-likelihood training, contextual Transformer states, and static embedding spaces \citep{gao2019representation,ethayarajh2019contextual,mu2018allbutthetop}. That literature characterizes representational sharing, anisotropy, and predictive quality. Our tied-embedding diagnostic instead traces how an output-side auxiliary source returns through both the output head and the input lexical path. \citet{oec2026} directly target common-shift instability in LLM pretraining through output embedding centering. We use centering both as a mitigation and as an intervention that exposes the gauge freedom of the output head. The reason is that CE is invariant to a common shift while Z-loss is not. The removed-source proposition identifies the exact softmax-shaped component removed by centering and explains why centering can help or introduce different trade-offs across models and training regimes.

\paragraph{Logit geometry and normalized heads.} A broad literature reduces or analyzes logit geometry through temperature scaling, label smoothing, confidence penalties, normalized heads, output-centering methods, final-logit soft-capping, activation normalizers, and studies of softmax output-layer rank \citep{guo2017calibration,szegedy2016rethinking,muller2019when,pereyra2017confidence,nguyen2019transformers,wei2022logitnorm,gemma2024,ba2016layer,zhang2019rmsnorm,yang2018softmax}. The shared objects are logit scale, shift, confidence, and expressivity. Relative to this literature, the distinction we draw is the coordinate of intervention. A method can reduce $\log Z$ or the centered relative log-normalizer while leaving $W_U^\top p$ and optimizer-facing update tails nearly unchanged. This coordinate distinction motivates the experimental diagnostics. We include $A_p^{99}$, gradient-event rates, Adam-state endpoints, and randomized low-rank unembedding audits \citep{halko2011randomized} rather than comparing scalar Z-loss values alone.

The common-shift coordinate isolated here is related to temperature calibration and logit-norm regularization \citep{guo2017calibration,wei2022logitnorm}. The common-shift coordinate is also related to embedding and activation normalization choices in Transformer training \citep{nguyen2019transformers,xiong2020layer,shleifer2021normformer}. The coordinate is nevertheless distinct from layer normalization and root-mean-square layer normalization. These methods normalize or recenter hidden states rather than removing the output-head common-shift gauge \citep{ba2016layer,zhang2019rmsnorm}.

\paragraph{Positioning of this study.} Taken together, the surrounding literature has treated Z-loss as a recipe-level mitigation for training-loss divergence, a router auxiliary term, a logit-geometry regularizer, or a numerical-loss implementation detail. This paper occupies the intersection of those views but focuses on the logit-space adjoint created after evaluation of the scalar log-normalizer penalty and on the transport of the logit-space adjoint through architecture, precision, routing reductions, and optimizer state. The contribution is therefore not a new sparse architecture, a new normalization layer, or a compute-scaling claim. Instead, the contribution is a transport-aware measurement, reporting, and intervention framework for existing dense and sparse Transformer recipes. The source identities are softmax-axis identities, but the evidence and recommendations intentionally target the Transformer regimes where Z-loss is deployed to constrain output-head or router log-normalizers and where the relevant transport paths arise together. By deriving source, common-shift, tied-path, gain, and router-scale identities and pairing those identities with transport diagnostics and validation PPL measured under matched run configurations, we make otherwise similar Z-loss reports comparable in the coordinates that affect updates. This coordinate-level comparison marks the point at which the results differ from the common scalar-control narrative. A method can improve the reported Z-loss value while failing to reduce the update coordinates that matter for mixed-precision headroom and adaptive-optimizer state; the relevant update coordinates can even increase. Accordingly, the proposed interventions have distinct roles. Centering, factorization, gain-aware weighting, and scale-matched router reductions act on particular transport coordinates and are not interchangeable ways to lower a scalar Z-loss curve.

\subsection{Raw-Coordinate Source and Adjoint}

Having fixed the CE-invariant gauge, we now compare the Z-loss sources produced by raw and centered deployed coordinates.

\begin{proposition}[Raw and centered Z-loss sources]
	Let $z=\widetilde z+\mu\mathbf{1}$ and let the raw and centered Z-losses use coefficients and targets $(\lambda_{\mathrm{raw}},c_{\mathrm{raw}})$ and $(\lambda_{\mathrm{ctr}},c_{\mathrm{ctr}})$. Since $\softmax(z)=\softmax(\widetilde z)=p$, the raw source and the centered source differ in the respective deployed-logit coordinates by
	\begin{equation}
		\begin{aligned}
			\delta_{z,\mathrm{raw}}^Z-\delta_{\widetilde z,\mathrm{centered}}^Z
			 & =
			2\lambda_{\mathrm{raw}}(\mu+\log\widetilde Z-c_{\mathrm{raw}})p   \\
			 & -2\lambda_{\mathrm{ctr}}(\log\widetilde Z-c_{\mathrm{ctr}})p .
		\end{aligned}
		\label{eq:removed_source_general}
	\end{equation}
\end{proposition}

\begin{corollary}[Removed deployed source component]
	When both losses use the same coefficient $\lambda$ and the same target $c$,
	\begin{equation}
		\delta_{z,\mathrm{raw}}^Z-\delta_{\widetilde z,\mathrm{centered}}^Z
		=
		2\lambda\mu p .
		\label{eq:removed_source}
	\end{equation}
\end{corollary}

\begin{proposition}[Projected raw-logit adjoint]
	Under the same single-coefficient, single-target convention as \Eqref{eq:removed_source}, if centered logits are implemented as $\widetilde z=Pz$ with $P\coloneqq I-(1/V)\mathbf{1}\mathbf{1}^\top$, $p=\softmax(Pz)=\softmax(z)$, and $\widetilde Z\coloneqq\sum_i\exp((Pz)_i)$, then the adjoint transported back to the raw-logit coordinate is
	\begin{equation}
		\nabla_z \mathcal{L}_Z(Pz)
		=
		2\lambda(\log\widetilde Z-c)
		\left(p-\frac{1}{V}\mathbf{1}\right).
		\label{eq:centered_raw_adjoint}
	\end{equation}
\end{proposition}

\begin{remark}[Coordinate interpretation]
	Thus, centering removes the common-shift contribution in the deployed-logit source coordinate. When mapped back through the centering projection, the optimizer-facing raw-logit adjoint is additionally projected to the zero-sum subspace. For CE, the deployed-logit source is $p-y$ and $\mathbf{1}^\top(p-y)=0$ for a one-hot target, so
	\begin{align*}
		\nabla_z\ell_{\mathrm{CE}}(Pz,y)=P(p-y)=p-y .
	\end{align*}
	Thus, the CE raw-logit adjoint is unchanged by the projection, whereas the Z-loss raw-logit adjoint changes because $Pp=p-(1/V)\mathbf{1}$. Hence, the removed common-shift component is specific to the Z-loss deployed source.
\end{remark}

Output-head centering preserves CE but intentionally changes the Z-loss source coordinate. The effect of centering should be judged by transport diagnostics such as $P_Z^{99.9}$ and $A_p^{99}$, with the complete diagnostic suite specified later, because removing the common-shift channel can be beneficial or introduce different trade-offs across model geometries.

\subsection{Factorized Objective}
\label{app:factorized_objective}

Writing $\nu\coloneqq\log\widetilde Z$ and choosing targets with $c=c_\mu+c_{\mathrm{rel}}$, the original Z-loss expands as
\begin{equation}
	\begin{aligned}
		\lambda(\mu+\nu-c)^2
		 & =
		\lambda(\mu-c_\mu)^2
		+2\lambda(\mu-c_\mu)(\nu-c_{\mathrm{rel}}) \\
		 & +\lambda(\nu-c_{\mathrm{rel}})^2 .
	\end{aligned}
	\label{eq:z_expand}
\end{equation}
\Eqref{eq:z_expand} motivates objective variants that target different transport coordinates. The factorized objective is a deliberately factorized inductive bias that drops the cross term and allows separate coefficients and targets for shift and relative-logit normalization.
\begin{equation}
	\mathcal{L}_{\mathrm{fact}}
	\coloneqq
	\lambda_\mu(\mu-c_\mu)^2
	+
	\lambda_{\mathrm{rel}}
	(\log\widetilde Z-c_{\mathrm{rel}})^2 .
	\label{eq:factorized}
\end{equation}

\subsection{Architecture-Aware Objectives}

For the gain-aware Z-loss runs, we use a detached output-to-hidden-gain weight with $\beta\ge 0$,
\begin{equation}
	\begin{aligned}
		A_p(z)
		 & \coloneqq
		\frac{\norm{W_U^\top p}_2}{\norm{p}_2+\varepsilon},
		\\
		\mathcal{L}_{\mathrm{gain}}
		 & \coloneqq
		\lambda
		\frac{(\log Z-c)^2}
		{1+\beta\sg(A_p(z)^2)} .
	\end{aligned}
	\label{eq:gain_aware}
\end{equation}
The stop-gradient is part of the objective used in the experiments; without the stop-gradient, differentiating $A_p(z)$ would add extra source terms. With the detached gain, the logit-coordinate source is
\begin{equation}
	\nabla_z\mathcal{L}_{\mathrm{gain}}
	=
	\frac{2\lambda(\log Z-c)}
	{1+\beta\sg(A_p(z)^2)}p .
	\label{eq:gain_aware_source}
\end{equation}

\subsection{Router Effective-Scale Derivations}
\label{app:router_scale_derivations}

This subsection gives the details of the router objective and reduction-scale statements summarized in the main text. Define
\begin{equation}
	\begin{aligned}
		Z_{R,t,\ell}
		 & \coloneqq
		\sum_{e=1}^{N_{\mathrm{exp}}}\exp r_{t,\ell,e}, \\
		\mathcal{L}_{R}
		 & \coloneqq
		\lambda_R
		\sum_{t,\ell}
		\alpha_{t,\ell}
		\left(
		\log Z_{R,t,\ell}-c_R
		\right)^2 .
	\end{aligned}
	\label{eq:router_loss}
\end{equation}

\begin{proposition}[Router coefficient and relative scale]
	Let $N_{\mathrm{tl}}$ be the number of token-layer router decisions, and suppose token-layer-mean router Z-loss uses $\alpha=1/N_{\mathrm{tl}}$. If an implementation uses a uniform per-decision coefficient $\lambda_R\alpha$, the effective gradient scale relative to the token-layer mean is $N_{\mathrm{tl}}\alpha$. More generally, decision $(t,\ell)$ has a relative scale of $N_{\mathrm{tl}}\alpha_{t,\ell}$.
	\label{prop:router_effective_coefficient}
\end{proposition}

For the top-$k$ convention statements below, write
\begin{align*}
	\phi_{t,\ell}
	\coloneqq
	\left(\log Z_{R,t,\ell}-c_R\right)^2 .
\end{align*}

\begin{corollary}[Top-$k$ active-route mean scale]
	\label{cor:router_topk_active_mean}
	Under the same token-layer-mean reference, assuming exactly $k$ active routes per token-layer decision, if one computes a single violation per token-layer decision, does not replicate the violation in the numerator, and nevertheless normalizes by the active-route count $N_{\mathrm{tl}}k$,
	\begin{align*}
		\mathcal{L}_R^{\mathrm{activeMean}}
		\coloneqq
		\lambda_R\frac{1}{N_{\mathrm{tl}}k}\sum_{t,\ell}\phi_{t,\ell},
	\end{align*}
	then the effective scale is $1/k$.
\end{corollary}

\begin{corollary}[Top-$k$ active-route sum scale]
	\label{cor:router_topk_active_sum}
	Under the same token-layer-mean reference, assuming exactly $k$ active routes per token-layer decision, if an implementation sums $k$ active-route replicas of the violation before applying the token-layer mean,
	\begin{align*}
		\mathcal{L}_R^{\mathrm{activeSum}}
		\coloneqq
		\lambda_R\frac{1}{N_{\mathrm{tl}}}\sum_{t,\ell}\sum_{j=1}^{k}\phi_{t,\ell}
		=
		\lambda_R\frac{k}{N_{\mathrm{tl}}}\sum_{t,\ell}\phi_{t,\ell},
	\end{align*}
	then the effective scale is $k$.
\end{corollary}

\paragraph{Top-$k$ conventions.} In the experiments and tables, active-route mean denotes the single-violation convention that normalizes by the active-route count. Active-route sum denotes the convention that sums active-route replicas before taking the token-layer mean in Corollary~\ref{cor:router_topk_active_sum}. Thus, the reported effective scale is $1/k$ for active-route mean by Corollary~\ref{cor:router_topk_active_mean}. The reported effective scale is $k$ for active-route sum by Corollary~\ref{cor:router_topk_active_sum}.

\subsection{Formal Details and Reference Implementation}

\subsubsection{Notation}

\begin{center}
	\captionof{table}{Notation summary}
	\label{tab:list_notation}
	\scriptsize
	\begin{tabularx}{\columnwidth}{@{}lY@{}}
		\toprule
		Symbol                                                          & Meaning                                                           \\
		\midrule
		$z,z_i$                                                         & Output-logit vector and a logit coordinate                        \\
		$V,\mathbf{1},y$                                                & Vocabulary size, ones vector, target                              \\
		$Z,\widetilde Z,\log Z$                                         & Raw and centered normalizers and the raw log-normalizer           \\
		$\mathcal{L}_Z$                                                 & Log-normalizer Z-loss                                             \\
		$\mathcal{L}_{\mathrm{fact}}$                                   & Factorized Z-loss                                                 \\
		$\mathcal{L}_{\mathrm{gain}}$                                   & Gain-aware Z-loss                                                 \\
		$c$                                                             & Z-loss target                                                     \\
		$\lambda,\lambda_{\rm aux},\lambda_{\rm diag}$                  & Dense Z-loss coefficients                                         \\
		\midrule
		$p,\widehat p$                                                  & Exact and approximate softmax probabilities                       \\
		$\ell_{\mathrm{CE}}$                                            & Token CE loss                                                     \\
		$\delta_z^Z,\widehat{\delta}_z^Z$                               & Exact and approximate Z-loss logit gradients (sources)            \\
		$\epsilon_Z,\epsilon_p,\varepsilon$                             & Source errors, stabilizer                                         \\
		$\theta,J_\theta$                                               & Parameters and logit Jacobian                                     \\
		$g_Z,g_{\mathrm{CE}}$                                           & Z-loss and CE gradients                                           \\
		$R_Z,C_Z$                                                       & Gradient ratio and cosine                                         \\
		$P_Z^{99.9}$                                                    & 99.9th-percentile (p99.9) log-normalizer tail                     \\
		$\Delta_{\mathrm{src}},\Delta_\theta$                           & Source and gradient mismatch                                      \\
		\midrule
		$W_U,E,u$                                                       & Output head, embedding, hidden state                              \\
		$\bar w,\bar b,\mu,\mu_{\mathrm{head}}$                         & Output-head common shift terms                                    \\
		$\widetilde W_U,\widetilde z,P$                                 & Centered head, logits, projection                                 \\
		$c_\mu,c_{\mathrm{rel}}$                                        & Shift and relative targets                                        \\
		$\lambda_\mu,\lambda_{\mathrm{rel}}$                            & Factorized coefficients                                           \\
		$A_p,a_t,\bar A_p,A_p^{99}$                                     & Output-to-hidden gain metrics                                     \\
		$\beta,\sg(\cdot)$                                              & Gain weight and stop-gradient                                     \\
		$\sigma_j,u_j,q$                                                & Singular value, left singular vector, normalized source direction \\
		\midrule
		$r_{t,\ell,e}$                                                  & Router logit                                                      \\
		$N_{\mathrm{exp}},N_{\mathrm{tl}},k$                            & Experts, decisions, top-$k$                                       \\
		$\lambda_R,\alpha_{t,\ell},\phi_{t,\ell}$                       & Router weight terms                                               \\
		$\lambda_{R,t,\ell}^{\mathrm{abs}},s_{R,t,\ell}^{\mathrm{rel}}$ & Router effective scales                                           \\
		$\delta_u^Z,\delta_h^Z$                                         & Hidden-state Z corrections                                        \\
		\midrule
		$G_{99},H_{99},\rho_{99}$                                       & Gradient-tail endpoints                                           \\
		$\Phi_\tau,\Phi_{\mathrm{nf}},\tau$                             & Gradient-event rates, threshold                                   \\
		$m_s,v_s,d_s^{\mathrm{proxy}}$                                  & Adam state and proxy direction                                    \\
		$M_{16},\kappa_{\max,b},\Psi_\kappa$                            & fp16 overflow-margin metrics                                      \\
		\midrule
		$m,S_{\mathrm{lse}}$                                            & Log-sum-exp statistics                                            \\
		$Q_\alpha,\E_t$                                                 & Quantile and token average                                        \\
		\bottomrule
	\end{tabularx}
\end{center}

\subsubsection{Proofs}

\paragraph{Z-loss common-shift sensitivity.}
\begin{proof}
	A direct computation gives
	\begin{align*}
		Z(z+a\mathbf{1}) & =e^a Z(z),
		\qquad
		\log Z(z+a\mathbf{1})=\log Z(z)+a,
	\end{align*}
	and therefore
	\begin{align*}
		\mathcal{L}_Z(z+a\mathbf{1})
		=\lambda(\log Z(z)+a-c)^2 .
	\end{align*}
	Differentiating this shifted objective with respect to $z_i$ while holding $a$ fixed yields
	\begin{align*}
		\frac{\partial}{\partial z_i}\mathcal{L}_Z(z+a\mathbf{1})
		=
		2\lambda(\log Z(z)+a-c)
		\frac{e^{z_i}}{\sum_j e^{z_j}},
	\end{align*}
	because $\partial \log Z(z)/\partial z_i=\softmax(z)_i$. Collecting coordinates gives the displayed gradient.
\end{proof}

\paragraph{Centering preserves CE.}
\begin{proof}
	By the definition of $\bar w$,
	\begin{align*}
		\widetilde z
		 & =\widetilde W_Uu           \\
		 & =W_Uu-\mathbf{1}(\bar w u) \\
		 & =z-\alpha\mathbf{1},
		\qquad
		\alpha\coloneqq\bar w u\in\R .
	\end{align*}
	Thus, centering subtracts the same scalar from every logit coordinate of the token. For each vocabulary index $i$,
	\begin{align*}
		\softmax(\widetilde z)_i
		=
		\frac{e^{z_i-\alpha}}{\sum_j e^{z_j-\alpha}}
		=
		\frac{e^{z_i}}{\sum_j e^{z_j}}
		=
		\softmax(z)_i .
	\end{align*}
	Since token CE with a one-hot target is the negative logarithm of the target softmax probability, the CE loss is identical for $z$ and $\widetilde z$. The argument relies only on exact algebra, hence the stated real-arithmetic qualification.
\end{proof}

\paragraph{Raw and centered Z-loss sources.}
\begin{proof}
	For $z=\widetilde z+\mu\mathbf{1}$, the two deployed-logit coordinates have the same softmax $p$, and $Z(z)=e^\mu\widetilde Z$, so $\log Z(z)=\mu+\log\widetilde Z$. Using \Eqref{eq:source}, the raw and centered Z-loss sources are therefore
	\begin{align*}
		\delta_{z,\mathrm{raw}}^Z
		 & =
		2\lambda_{\mathrm{raw}}(\mu+\log\widetilde Z-c_{\mathrm{raw}})p, \\
		\delta_{\widetilde z,\mathrm{centered}}^Z
		 & =
		2\lambda_{\mathrm{ctr}}(\log\widetilde Z-c_{\mathrm{ctr}})p .
	\end{align*}
	Subtracting the two expressions gives \Eqref{eq:removed_source_general}.
\end{proof}

\paragraph{Removed deployed source component.}
\begin{proof}
	Set $\lambda_{\mathrm{raw}}=\lambda_{\mathrm{ctr}}=\lambda$ and $c_{\mathrm{raw}}=c_{\mathrm{ctr}}=c$ in \Eqref{eq:removed_source_general}.
\end{proof}

\paragraph{Projected raw-logit adjoint.}
\begin{proof}
	If the centered loss is viewed as a function of the raw logits through $\widetilde z=Pz$, then the chain rule gives
	\begin{align*}
		\nabla_z \mathcal{L}_Z(Pz)
		=
		P^\top\delta_{\widetilde z,\mathrm{centered}}^Z .
	\end{align*}
	Since $P$ is symmetric and $Pp=p-(1/V)\mathbf{1}$, \Eqref{eq:centered_raw_adjoint} follows.
\end{proof}

\paragraph{Tied embedding gradient decomposition.}
\begin{proof}
	Let $\Phi(E_{\mathrm{in}},E_{\mathrm{out}})$ denote the scalar Z-loss of the untied network and define the tied objective by $\phi(E)\coloneqq\Phi(E,E)$. For any perturbation $H$ of the tied embedding table,
	\begin{align*}
		D\phi(E)[H]
		=
		D_{E_{\mathrm{in}}}\Phi(E,E)[H]
		+
		D_{E_{\mathrm{out}}}\Phi(E,E)[H],
	\end{align*}
	which gives the displayed decomposition after identifying differentials with Frobenius inner products.
\end{proof}

\paragraph{Gain is anisotropy-weighted.}
\begin{proof}
	Let $W_U=U\Sigma V^\top$ be a singular value decomposition, with left singular vectors $u_j$ and singular values $\sigma_j\ge 0$. Since $q=p/\norm{p}_2$ and $\varepsilon$ is ignored,
	\begin{align*}
		a(p)^2
		=
		\frac{\norm{W_U^\top p}_2^2}{\norm{p}_2^2}
		=
		\norm{W_U^\top q}_2^2 .
	\end{align*}
	Using $W_U^\top=V\Sigma^\top U^\top$ and the orthonormality of the columns of $V$,
	\begin{align*}
		\norm{W_U^\top q}_2^2
		=
		\norm{\Sigma^\top U^\top q}_2^2
		=
		\sum_j \sigma_j^2 \inner{q}{u_j}^2 .
	\end{align*}
	Because the $u_j$ form an orthonormal set, $\sum_j\inner{q}{u_j}^2\le \norm{q}_2^2=1$. Hence,
	\begin{align*}
		a(p)^2
		\le
		\sigma_{\max}(W_U)^2\sum_j\inner{q}{u_j}^2
		\le
		\sigma_{\max}(W_U)^2 .
	\end{align*}
\end{proof}

\paragraph{Router coefficient and relative scale.}
\begin{proof}
	For a fixed token-layer router decision, define
	\begin{align*}
		\phi_{t,\ell}
		\coloneqq
		\left(\log Z_{R,t,\ell}-c_R\right)^2 .
	\end{align*}
	In the token-layer-mean convention,
	\begin{align*}
		\mathcal{L}_R^{\mathrm{tok}}
		\coloneqq
		\lambda_R\frac{1}{N_{\mathrm{tl}}}\sum_{t,\ell}\phi_{t,\ell},
	\end{align*}
	so the coefficient multiplying each decision is $\lambda_R/N_{\mathrm{tl}}$. Differentiating \Eqref{eq:router_loss} for a fixed decision shows that the router Z-loss gradient is linear in the absolute per-decision coefficient $\lambda_R\alpha_{t,\ell}$. Thus, the scale relative to the token-layer mean is
	\begin{align*}
		\frac{\lambda_R\alpha}{\lambda_R/N_{\mathrm{tl}}}=N_{\mathrm{tl}}\alpha .
	\end{align*}
\end{proof}

\paragraph{Top-$k$ active-route mean scale.}
\begin{proof}
	For the convention that divides a single violation by the active-route count, $\alpha=1/(N_{\mathrm{tl}}k)$, so Proposition~\ref{prop:router_effective_coefficient} gives $N_{\mathrm{tl}}\alpha=1/k$.
\end{proof}

\paragraph{Top-$k$ active-route sum scale.}
\begin{proof}
	For active-route sum before the token-layer mean, the per-decision coefficient is $k\lambda_R/N_{\mathrm{tl}}$, equivalently $\alpha=k/N_{\mathrm{tl}}$, so Proposition~\ref{prop:router_effective_coefficient} gives $N_{\mathrm{tl}}\alpha=k$.
\end{proof}

\subsubsection{Reference Python Implementation}

Listing~\ref{lst:centered_source} gives an example implementation of the dense output-head operations used in the experiments. The first function deploys centered logits by subtracting only the fp32 output-head common shift from already computed raw logits, matching \Eqref{eq:centered_raw}. The second function returns the Z-loss logit adjoint from \Eqref{eq:source}, including the mean-reduction scaling used by the training and audit code.

\begin{lstlisting}[float=t!,style=pythoncompact,caption={Python implementation example of centered deployed logits and Z-loss logit-source construction.},label={lst:centered_source}]
import torch
import torch.nn.functional as F


def centered_logits_from_raw(raw_logits, hidden, weight, bias=None):
    # Subtract the output-head common shift.
    mean_w = weight.float().mean(dim=0, keepdim=True)
    mean_b = None if bias is None else bias.float().mean().view(1)
    common_shift = F.linear(hidden.float(), mean_w, mean_b)
    return raw_logits.float() - common_shift


def z_loss_source(logits, coef, target, reduction="mean"):
    # Gradient of coef * (logsumexp(logits) - target)^2.
    if reduction not in {"mean", "sum", "none"}:
        raise ValueError(f"unknown reduction: {reduction}")
    logits_f = logits.float()
    log_z = torch.logsumexp(logits_f, dim=-1, keepdim=True)
    probs = torch.softmax(logits_f, dim=-1)
    scale = 2.0 * coef * (log_z - target)
    if reduction == "mean":
        scale = scale / logits_f[..., 0].numel()
    return scale * probs
\end{lstlisting}

\subsection{Diagnostics, Interventions, and Coefficient Regimes}
\label{app:diagnostic_conventions}

\subsubsection{Diagnostic Decision Framework}

We use diagnostics as a decision hierarchy rather than as a leaderboard. A diagnostic first identifies an active source or transport coordinate; the corresponding intervention targets that coordinate; and optimizer-facing endpoints, together with validation PPL, determine whether the intervention is acceptable. Appendix Table~\ref{tab:diagnostic_intervention_map} gives the complete mapping, including the effects that each intervention does and does not guarantee.

At the source-coordinate level, p99 $|\mu|$ diagnoses common-shift excursions. Row-centered deployment sets this coordinate to zero while preserving the softmax distribution and CE, but row-centered deployment does not guarantee reductions in $P_Z^{99.9}$ or $A_p^{99}$; those remain geometry-dependent outcomes. Centered Z-loss additionally controls the log-normalizer in the remaining relative-logit coordinate. The factorized objective in \Eqref{eq:factorized} is used as a controlled decomposition of the shift and relative-logit coordinates rather than as a universally preferred replacement for raw or centered Z-loss.

At the architectural-transport level, $A_p^{99}$ diagnoses output-to-hidden amplification. Gain-aware weighting in \Eqref{eq:gain_aware} attenuates the source conditional on that gain and is evaluated by hidden-source and optimizer-facing tails; gain-aware weighting is not expected to reduce the $A_p^{99}$ metric. For tied embeddings, output- and input-path gradient norms and cosines provide a pathway audit rather than a separate intervention. At the implementation level, $\Delta_{\mathrm{src}}$ and $\Delta_\theta$ diagnose forward--backward source inconsistency. Reusing forward log-sum-exp statistics removes reconstruction inconsistency, although reuse cannot recover precision already lost through quantized logit storage. For MoE routers, $s_R^{\mathrm{rel}}=N_{\mathrm{tl}}\alpha$ is a deterministic scale audit, and coefficients are matched whenever $s_R^{\mathrm{rel}}\ne 1$.

We distinguish these root-cause diagnostics from downstream endpoints. Here, Diag. Z is computed on the deployed logits, and $P_Z^{99.9}\coloneqq Q_{0.999,t}(|\log Z_t-c|)$ measures the corresponding log-normalizer tail; neither is a coordinate-free score across raw and centered deployments. The gain tail $A_p^{99}$ is defined in \Eqref{eq:ap}. $R_Z$ and $R_{\rm aux}$ describe auxiliary-gradient strength relative to CE, while $C_Z$ and $C_{\rm aux}$ describe CE alignment; the gradient ratios and cosines are descriptive rather than quantities every intervention should minimize. Pre-clip gradient events, Adam-state tails, and static fp16 overflow margins measure optimizer-facing consequences. No tail reduction alone is sufficient. An intervention is accepted only when the diagnosed coordinate improves and validation quality remains acceptable.

The low-coefficient continued-pretraining regime uses $\lambda_{\rm aux}=1\times10^{-4}$ to study geometry with limited over-regularization. High-coefficient stress settings use $\lambda_{\rm aux}\in\{1\times10^{-3},3\times10^{-3}\}$ together with stronger optimization pressure.

The same intervention can be quality-neutral at low auxiliary coefficients and reveal trade-offs between validation quality and both update-tail behavior and mixed-precision headroom under stress. The experiments therefore pair validation quality with log-normalizer, gain, gradient-event, optimizer-state, and overflow-headroom endpoints.

\subsubsection{Diagnostic Metrics}

We report a common diagnostic suite. Here and below, $Q_{\alpha,t}$ denotes the empirical $\alpha$-quantile over tokens; the reported tail summaries use p99 and p99.9.
\begin{align}
	R_Z
	 & \coloneqq
	\frac{\norm{g_Z}_F}{\norm{g_{\mathrm{CE}}}_F+\varepsilon},\notag \\
	P_Z^{99.9}
	 & \coloneqq
	Q_{0.999,t}\left(|\log Z_t-c|\right),                            \\
	C_Z
	 & \coloneqq
	\frac{\inner{g_Z}{g_{\mathrm{CE}}}}
	{\norm{g_Z}_F\norm{g_{\mathrm{CE}}}_F+\varepsilon},              \\
	\bar A_p
	 & \coloneqq
	\E_t a_t,
	\qquad
	A_p^{99}\coloneqq Q_{0.99,t}(a_t),                               \\
	\Delta_{\mathrm{src}}
	 & \coloneqq
	\frac{\norm{\widehat{\delta}_z^Z-\delta_z^Z}_2}
	{\norm{\delta_z^Z}_2+\varepsilon},                               \\
	\Delta_{\theta}
	 & \coloneqq
	\frac{\norm{g_{Z,\theta}^{\mathrm{approx}}-g_{Z,\theta}^{\mathrm{ref}}}_F}
	{\norm{g_{Z,\theta}^{\mathrm{ref}}}_F+\varepsilon},              \\
	\lambda_{R,t,\ell}^{\mathrm{abs}}
	 & \coloneqq
	\lambda_R\alpha_{t,\ell},
	\qquad
	s_{R,t,\ell}^{\mathrm{rel}}\coloneqq N_{\mathrm{tl}}\alpha_{t,\ell}.
\end{align}

\begin{table*}[t!]
	\centering
	\scriptsize
	\begin{tabularx}{\textwidth}{@{}YYYYYYY@{}}
		\toprule
		Stage & Risk                                                                                                        & Primary trigger or audit & Action & Expected signature & Not guaranteed & Evidence \\
		\midrule
		Source coordinate
		      & Common-shift excursion
		      & p99 $|\mu|$
		      & Row-centered deployment
		      & $\mu=0$ with unchanged softmax and CE
		      & Lower $P_Z^{99.9}$, $A_p^{99}$, or PPL
		      & Tables~\ref{tab:pretrained}, \ref{tab:finetune}, and Appendix Table~\ref{tab:fineweb_pretrained}                                                                                                 \\
		\addlinespace
		Objective decomposition
		      & Shift and relative-logit coordinates require separate control
		      & Component-wise $\mu$ and $\log\widetilde Z$ violations
		      & Factorized Z-loss
		      & Separate coefficients and targets with the cross term removed
		      & Lower update tails or better PPL; the factorized objective is a controlled probe
		      & \Eqref{eq:factorized} and Appendix Table~\ref{tab:fineweb_update_outlier_sweep}                                                                                                                  \\
		\addlinespace
		Architectural transport
		      & Large output-to-hidden gain
		      & $A_p^{99}$ and hidden-source tails
		      & Detached gain-aware weighting
		      & Attenuated source and update tails conditional on $A_p$
		      & A lower unweighted $A_p^{99}$ value
		      & Table~\ref{tab:stress_diag} and Appendix Tables~\ref{tab:spectral_audit} and~\ref{tab:adam_state_stress}                                                                                         \\
		\addlinespace
		Shared pathway
		      & Tied input and output coupling
		      & Input- and output-path norms and cosines
		      & Pathway audit and coefficient review
		      & Coupled update made visible at the shared table
		      & Automatic gradient reduction; no separate intervention is claimed
		      & Tied-path diagnostics in the Appendix                                                                                                                                                            \\
		\addlinespace
		Implementation
		      & Forward--backward source mismatch
		      & $\Delta_{\mathrm{src}}$ and $\Delta_\theta$
		      & Reuse forward log-sum-exp statistics
		      & Backward source consistent with the computed forward loss
		      & fp32 fidelity after quantized logit storage
		      & Appendix Tables~\ref{tab:source_audit}, \ref{tab:triton_fused}, and~\ref{tab:update_audit}                                                                                                       \\
		\addlinespace
		Router reduction
		      & Effective coefficient differs across top-$k$ conventions
		      & $\lambda_R\alpha$ and $s_R^{\mathrm{rel}}=N_{\mathrm{tl}}\alpha$
		      & Match the coefficient to $s_R^{\mathrm{rel}}=1$
		      & Token-layer-mean correction scale recovered
		      & Better validation quality or load balance
		      & Appendix Tables~\ref{tab:router} and~\ref{tab:moe}                                                                                                                                               \\
		\addlinespace
		Optimizer consequence
		      & Rare transported update pressure
		      & Pre-clip tails, $\Phi_\tau$, Adam state, and fp16 margin
		      & Apply the diagnosed root-cause action, then re-audit
		      & Relevant update or numerical tail reduced
		      & Acceptable validation quality
		      & Appendix Tables~\ref{tab:fineweb_update_outlier}, \ref{tab:adam_state_stress}, and~\ref{tab:overflow_audit}                                                                                      \\
		\addlinespace
		Quality gate
		      & Geometry improves at a quality cost
		      & Validation CE and PPL
		      & Reject, retune, or change the intervention
		      & Quality remains within an explicitly stated tolerance
		      & Simultaneous improvement of every diagnostic
		      & Matched continued-pretraining and coefficient sweeps                                                                                                                                             \\
		\bottomrule
	\end{tabularx}
	\caption{Diagnostic decision map. A trigger identifies a source, transport, implementation, or reduction coordinate; the action targets that coordinate; and consequence endpoints plus validation quality determine whether the action is acceptable. The Not guaranteed column records quantities that should not be used as automatic success criteria for that action.}
	\label{tab:diagnostic_intervention_map}
\end{table*}

\subsubsection{Transport and Numerical Endpoints}

For transport and numerical endpoints, we report optimizer-facing gradient tails, thresholded gradient-event rates, Adam-state endpoints, and static fp16 loss-scale headroom. We do not rely only on final validation quality. Here, pre-clip means measured before any gradient clipping is applied. We use $\Phi_\tau$, with table columns such as $\Phi_{10}$, $\Phi_{50}$, and $\Phi_{100}$, to denote the fraction of optimizer steps whose pre-clip full-model gradient norm exceeds threshold $\tau$. These endpoints, including the Adam-state and static fp16 overflow proxies, are formally defined in the Transport and Numerical Endpoint Definitions subsection.

Gradient clipping is complementary to these endpoint measurements. Gradient clipping can cap the post-backward update norm and is an appropriate training guardrail, but gradient clipping does not change the logit-space Z-loss source, the architecture-dependent transport factor in \Eqref{eq:transport_bound}, or source-consistency errors in fused and low-precision paths. We therefore report pre-clip tails together with clipping or threshold-event frequencies. A method may follow a different effective optimization trajectory when the apparent gradient-tail control of the method depends on frequent clipping. Adam-state endpoints should be read under the stated clipping convention. Clip-before-Adam stacks can reduce moment-state pollution, whereas effectively unclipped stress settings expose the unmitigated transported source.

\subsubsection{Interventions and Coefficient Regimes}

The corresponding practical interventions follow directly. Implementations compute the forward log-sum-exp sufficient statistics $m\coloneqq\max_i z_i$ and $S_{\mathrm{lse}}\coloneqq\sum_i\exp(z_i-m)$ once and reuse the statistics during the backward pass, mirroring the sufficient-statistic discipline behind online softmax kernels \citep{milakov2018online}. This shared-statistics discipline extends standard numerical softmax stabilization to the Z-loss branch and ensures that fused or low-precision paths transport the backward source implied by the forward loss. Row-centered output-head deployment fixes the CE-invariant output-head common-shift gauge and removes that coordinate from Z-loss transport. The row-centered deployment is a model-coordinate intervention, not the internal row-maximum subtraction used to evaluate softmax stably. In experimental tables, Centered head denotes row-centered deployed logits obtained by subtracting the output-head common shift from raw logits, with no applied auxiliary loss. The Centered head configuration is a gauge-only control whose exact-arithmetic CE gradient equals that of the raw CE baseline; the configuration's Z-loss metrics characterize the deployed coordinate rather than an applied auxiliary update. Centered Z-loss denotes the same deployed-logit convention with an applied centered Z-loss. Factorized Z-loss denotes \Eqref{eq:factorized} applied to raw logits with $\lambda_\mu=\lambda_{\mathrm{rel}}=\lambda_{\rm aux}$, $c_\mu=0$, and $c_{\mathrm{rel}}=\log V$ unless stated otherwise. Gain-aware Z-loss denotes \Eqref{eq:gain_aware} applied to raw logits with the detached gain weight. Stress and overflow rows use $\beta=1\times 10^{-2}$ unless stated otherwise. Comparable router reports include both the absolute coefficient $\lambda_R\alpha$ and the scale $N_{\mathrm{tl}}\alpha$ relative to the token-layer mean rather than only nominal $\lambda_R$.

The dense output-head experiments use two coefficient regimes to separate geometry from stress behavior. The low-coefficient continued pretraining regime uses an applied auxiliary coefficient of $\lambda_{\rm aux}=1\times 10^{-4}$. Several matched settings have small paired PPL differences from CE, allowing differences in $P_Z^{99.9}$ and $A_p^{99}$ to track geometry with limited over-regularization. We report the paired differences rather than treat overlap of three-run standard deviations as a formal equivalence test. The high-coefficient stress regime uses $\lambda_{\rm aux}\in\{1\times 10^{-3},3\times 10^{-3}\}$ and higher learning rates. This regime isolates scalar-tail reduction from transported update tails under stronger auxiliary pressure. These settings are interpreted through transport and numerical endpoints together with paired PPL differences. In tables, Standard Z-loss denotes \Eqref{eq:zloss} with $c=\log V$ applied to raw logits; conventional formulations commonly use $c=0$. Diag. Z uses the reporting coefficient $\lambda_{\rm diag}$ specified for that table or run family. Unless otherwise specified in a caption, dense Diag. Z tables use $\lambda_{\rm diag}=1\times 10^{-4}$. For CE-only or Centered head rows, $\lambda_{\rm aux}=0$ and Diag. Z is not an applied auxiliary loss. In table columns, $R_{\rm aux}$ denotes the auxiliary-gradient norm divided by the CE-gradient norm. $C_{\rm aux}$ denotes the auxiliary-to-CE gradient cosine.

The Z-loss target is $c=\log V$ unless stated otherwise. The zero logit vector, equivalently a zero-mean uniform logit vector, therefore has zero violation. For centered relative-logit diagnostics, the relative target uses the same $\log V$ reference after removing the token-wise mean. We report high-percentile statistics, especially p99 and the 99.9th percentile (p99.9), because the hypothesized failure mode is rare-token or rare-step numerical pressure. Mean Z-loss can improve while the tail of $J^\top\delta_z^Z$ worsens. We therefore treat $P_Z^{99.9}$, $A_p^{99}$, thresholded $\Phi_\tau$, maximum-coordinate tails, Adam-state tails, and static loss-scale endpoints as primary transport and numerical measurements. PPL is used to reject interventions that buy tail reduction by degrading the language model.

\subsection{Experimental Protocol}

\subsubsection{Experimental Setup Details}
\label{app:setup_details}

\paragraph{Pretrained diagnostics.} Main diagnostics use corpus text and pretrained language models. The evaluated models are DistilGPT-2, GPT-2, GPT-2 Medium, and GPT-2 Large \citep{radford2019language,wolf2020transformers}, together with Pythia-160M and Pythia-410M \citep{biderman2023pythia}. For reproducibility, the latter are the \texttt{EleutherAI/pythia-160m-deduped} and \texttt{EleutherAI/pythia-410m-deduped} checkpoints, both trained on the deduplicated Pile. The primary evaluation uses non-overlapping WikiText-103 validation blocks of length 256 \citep{merity2016pointer}. DistilGPT-2 and GPT-2 use 512 validation sequences, yielding 131,072 next-token predictions. GPT-2 Medium and Pythia-160M use 256 validation sequences, yielding 65,536 predictions. GPT-2 Large and Pythia-410M use 128 validation sequences, yielding 32,768 predictions. These smaller evaluations target scale-sensitive cases where centering can move tail metrics in either direction. The diagnostics also include a 64-block WikiText evaluation of \texttt{EleutherAI/pythia-1b-deduped}. For compact table labels, Pythia-160M, Pythia-410M, and Pythia-1B denote these checkpoints throughout. To test whether the diagnostic depends on WikiText, we also evaluate GPT-2 and GPT-2 Medium on FineWeb-Edu. This cross-corpus check uses 128 streamed blocks from the FineWeb-Edu \texttt{sample-10BT} training sample \citep{penedo2024fineweb}. The cross-corpus check is paired with 64-block FineWeb-Edu evaluations for GPT-2 Large and Pythia-410M.

\paragraph{Dense training protocols.} Dense training interventions continue pretraining GPT-2 from the same initialization on WikiText-103 training blocks. The main run uses 4800 optimizer steps, a micro-batch size of 16, a gradient-accumulation factor of 2, a sequence length of 256, and bf16 autocast. Optimization uses Adam with decoupled weight decay (AdamW), $\beta_1=0.9$, $\beta_2=0.95$, a weight decay of 0.1, a peak learning rate of $5\times 10^{-5}$, 240 warmup steps, cosine decay, and gradient clipping at 1.0 \citep{pascanu2013difficulty}. Here and below, the micro-batch size denotes the reported number of sequences in each forward and backward accumulation pass before gradient accumulation. For example, a micro-batch size of 16 with a gradient-accumulation factor of 2 gives an optimizer-step batch size of 32 sequences. Each 4800-step run consumes 39.3M training tokens and is evaluated on 131,072 validation tokens. Run counts and matched-run rules appear in the subsection on seed and matched-run details. A 1200-step GPT-2 diagnostic comparison uses the same three-run design, 9.83M tokens per run, and 65,536 validation tokens. This diagnostic comparison covers CE, standard Z-loss, and centered Z-loss. For this comparison, we measure p99 pre-clip full-model gradient norms and the tied-embedding auxiliary-to-CE gradient ratio, and we decompose the tied-embedding auxiliary gradient into output and input paths for the auxiliary gradient. A high-coefficient gradient-tail stress experiment uses $\lambda_{\rm aux}=1\times 10^{-3}$ for auxiliary-method rows. The stress experiment uses 1200 steps, 9.83M tokens per run, and the same matched-run design. In the stress experiment, we compare CE, centered head, standard Z-loss, centered Z-loss, and gain-aware Z-loss under the same optimizer. For CE and centered head rows, $\lambda_{\rm aux}=0$. The diagnostic coefficient $\lambda_{\rm diag}$ for the CE and centered head rows is used only for Diag. Z reporting.

\paragraph{Stress and precision endpoints.} For all continued pretraining runs, model parameters are kept as fp32 master weights. bf16 is used only through autocast. This precision convention prevents optimizer failures caused by fp16 storage from being conflated with the Z-loss geometry under study. The FineWeb-Edu GPT-2 Medium evaluation includes a 900-step effectively unclipped web-corpus update-tail setting. The 900-step setting uses the same train and validation offsets, a learning rate of $2\times 10^{-4}$, $\lambda_{\rm aux}=1\times 10^{-3}$, 3.69M consumed tokens per run, and a clipping threshold of $1\times 10^{8}$. This threshold effectively disables clipping. The setting probes update-tail behavior under aggressive optimization. An update-outlier setting uses the same model, dataset, and offsets. The update-outlier setting uses a learning rate of $8\times 10^{-4}$, $\lambda_{\rm aux}=3\times 10^{-3}$, 600 optimizer steps, 2.46M consumed tokens per run, and the same clipping threshold. During this experiment, we separately log the fractions of optimizer steps for which the pre-clip full-model gradient norm exceeds each of the thresholds 10, 50, and 100, together with the nonfinite-microbatch and nonfinite-gradient fractions. A centered auxiliary-coefficient sweep uses $\lambda_{\rm aux}\in\{1\times 10^{-3},3\times 10^{-4},1\times 10^{-4}\}$ under the same matched-run setup, data offsets, optimizer, and clipping threshold. An Adam-state endpoint in the same high-coefficient FineWeb-Edu setting samples AdamW state every 50 optimizer steps. This endpoint tests whether raw-logit transported spikes enter the second-moment memory of Adam even when no run reaches a nonfinite-gradient event.

\paragraph{Implementation audits.} For fused-kernel sensitivity, we implement a row-wise GPU kernel in \mbox{Triton} \citep{tillet2019triton} that fuses CE and Z-loss forward and backward computations. The fused kernel reports both total CE+Z logit-gradient consistency and Z-loss source consistency with a PyTorch fp32 reference \citep{paszke2019pytorch}. Large-vocabulary CE kernels increasingly fuse, approximate, or avoid materializing the full logit matrix. This audit therefore targets the same log-sum-exp source statistics that such implementations must preserve \citep{grave2017efficient,wijmans2024cut}. We run the fused-kernel audit on eight non-overlapping GPT-2 blocks from WikiText-103 and eight streamed GPT-2 blocks from FineWeb-Edu. The audit uses logit scales 1 and 4 with fp32, bf16, and fp16 logit storage. To test whether source mismatch reaches actual optimizer inputs, an update audit backpropagates total CE+Z logit sources computed from fp32 reference logits and storage-quantized logits through GPT-2 and GPT-2 Medium model graphs on corpus batches. In the update audit, we compare output-head gradients, whole-model gradients, and first-step Adam directions. The main output-head audit uses WikiText-103 and streamed FineWeb-Edu blocks, logit scales 1 and 4, and the same fp32 model parameters as the training runs. The whole-model audit materializes all transported gradients on smaller batches with the same corpus-derived logits. This whole-model audit confirms that the mismatch reaches beyond the output-head block.

\paragraph{Overflow-margin audit.} The overflow audit is a targeted mixed-precision overflow-margin endpoint. The overflow audit separates loss-scale headroom from optimizer-state dynamics by computing full-model gradients on held-out corpus text under bf16 autocast. In the overflow audit, we measure the maximum-coordinate gradient, the maximum safe static fp16 loss scale $\kappa_{\max,b}$ from \Eqref{eq:overflow_proxy}, and overflow batch fractions for static scales $\{256,1024,4096,16384,65536\}$. We run this audit on 16 WikiText-103 GPT-2 validation blocks for CE, standard Z-loss, and centered Z-loss. We also run the overflow audit on 16 streamed FineWeb-Edu GPT-2 Medium blocks for CE, centered head, standard Z-loss, factorized Z-loss, gain-aware Z-loss, and centered Z-loss. The audit uses $\lambda_{\rm aux}=3\times 10^{-3}$ for auxiliary-method rows to match the update-outlier setting.

\paragraph{Run accounting.} We separate validation-quality comparisons from paired diagnostic comparisons. The experiment descriptions distinguish available from consumed training tokens and report token reuse when it is relevant to the comparison. This distinction keeps fixed-subset evaluations separate from fresh-corpus continued pretraining and compute- and data-scale comparisons \citep{kaplan2020scaling,hoffmann2022training}. Here, the consumed-token count means training-token instances processed by optimizer steps, including repeated exposures to the same available tokens. For each matched run pair, we report the target-minus-baseline difference. A negative paired difference indicates improvement in PPL, CE, Diag. Z, tail, gain, and gradient-norm metrics.

\paragraph{Run-level instability analysis.} Because the failure modes under study are rare and optimizer-facing, we characterize training behavior using more than the mean loss. The reported endpoints include training-loss tails, the step-to-step loss-spike rate, thresholded pre-clip gradient-event rates, pre- and post-clip maximum-coordinate and gradient-to-parameter tails, finite-only gradient-tail summaries, and nonfinite-microbatch and nonfinite-gradient fractions. Paired comparisons hold the seed, precision mode, optimizer hyperparameters, gradient-clipping threshold, effective batch size, data offsets, and token accounting fixed. Related deterministic audits use analogous experiment-specific controls. This design helps separate differences in data exposure, precision, or clipping from the transport, optimizer-state, and numerical effects of an intervention.

\paragraph{FineWeb-Edu runs.} FineWeb-Edu continued pretraining experiments extend the dense training evaluation beyond WikiText. The GPT-2 experiment uses the same optimizer hyperparameters, 2400 steps, an optimizer-step batch size of 32 sequences, a sequence length of 256, 120 warmup steps, and bf16 autocast. Each run consumes 19.7M training tokens from streamed \texttt{sample-10BT} blocks. Each run is evaluated on 131,072 disjoint next-token predictions from the same split with a 70,000-sequence offset. No validation block overlaps the training block set. The GPT-2 Medium experiments include a 2400-step disjoint-block run with a micro-batch size of 8, a gradient-accumulation factor of 2, and 120 warmup steps. Each run in that experiment uses 9.83M training tokens and 65,536 validation next-token predictions. A fresh-token experiment uses 9600 steps, a micro-batch size of 8, a gradient-accumulation factor of 2, and 480 warmup steps. The fresh-token experiment uses 160,000 training sequences, 41.0M available tokens, 39.3M consumed tokens per run, and 1024 validation sequences offset by 220,000 sequences. This configuration yields 262,144 next-token predictions. The token-reuse factor is 0.96, so this experiment remains fresh-token within the available training blocks.

\paragraph{Transfer runs.} To test model-scale transfer, we continue pretraining GPT-2 Medium for 2400 optimizer steps. The run uses a micro-batch size of 8, a gradient-accumulation factor of 2, 120 warmup steps, and the same optimizer recipe. Each GPT-2 Medium run consumes 9.83M training tokens and is evaluated on 65,536 validation tokens. A 4800-step GPT-2 Medium fixed-subset evaluation uses the same micro-batch size, gradient-accumulation factor, and optimizer with 240 warmup steps. Each fixed-subset run consumes 19.7M tokens while repeatedly revisiting a fixed 8192-sequence training subset. This fixed-subset setting measures matched-mechanism behavior under controlled token reuse.

A Pythia-160M WikiText-103 continued pretraining grid tests whether the conclusions for the GPT-2 family transfer to a different model family with a padded output vocabulary and an untied output projection. We use 1200 optimizer steps, a micro-batch size of 8, a gradient-accumulation factor of 2, and 120 warmup steps. Each run consumes 4.9M tokens from 4.2M available training tokens and is evaluated on 65,536 validation tokens. For this grid, the Z-loss target is computed from the actual output-head vocabulary size rather than the tokenizer length.

\paragraph{Gradient-geometry visualization.} Figure~\ref{fig:gradient_geometry} is generated by a deterministic fp32 audit on the first 32 non-overlapping WikiText-103 validation blocks, using pretrained GPT-2, a block length of 256, an evaluation batch size of 2, a random seed of 67, $\lambda_{\rm diag}=10^{-4}$, and $c=\log V$. The audit computes the exact token-level logit-source and hidden-injection norms, the fixed translation induced by the output-row mean, and randomized rank-16 decompositions of both the raw and row-centered unembedding matrices with four power iterations. The paired plane uses the normalized output-row mean as the first axis and the leading principal direction of the centered hidden-source cloud after projecting out that mean direction as the second axis.

\paragraph{Spectral and router audits.} We evaluate the output-to-hidden gain identity with a spectral diagnostic. The spectral diagnostic computes randomized rank-16 singular value decompositions \citep{halko2011randomized} of the raw and row-centered unembedding matrices. The spectral diagnostic evaluates corpus-derived text logits and decomposes token-level $a_t^2=\norm{W_U^\top p_t}_2^2/\norm{p_t}_2^2$ into the contribution of the top singular directions. For this analysis, we compute the top-16 unembedding energy fraction, the mean token-level fraction of $a_t^2$ explained by the top singular directions, and the Pearson correlation between the low-rank contribution and exact $a_t^2$. We run the spectral diagnostic on GPT-2, GPT-2 Medium, Pythia-160M, and Pythia-1B WikiText-103 validation blocks. We also run the spectral diagnostic on GPT-2 Medium, GPT-2 Large, and Pythia-410M FineWeb-Edu blocks. These measurements estimate dominant-direction transport without modifying training.

For sparse routing, we run two experiments. We apply a randomly initialized 16-expert router to GPT-2 hidden states obtained from corpus blocks. This router projection isolates the reduction convention without changing the hidden-state distribution. We also train controlled end-to-end top-$2$ and top-$4$ MoE language models on WikiText-103. The top-$2$ model runs for 800 steps, and the top-$4$ model runs for 1200 steps. Both use three matched random initializations. The MoE has four Transformer blocks, a width of 256, four attention heads, eight feed-forward experts, tied input and output embeddings, and a standard load-balancing loss with a coefficient of $1\times 10^{-2}$. Router Z-loss uses $\lambda_R=1\times 10^{-3}$.

\subsubsection{Transport and Numerical Endpoint Definitions}

In addition to final validation quality, we report optimizer-facing gradient tails for the transport and numerical endpoints, as follows.
\begin{align}
	G_{99}
	 & \coloneqq
	Q_{0.99}\left(\norm{g_{\mathrm{train},s}}_2\right),\notag \\
	H_{99}
	 & \coloneqq
	Q_{0.99}\left(\norm{g_{\mathrm{train},s}}_\infty\right),  \\
	\rho_{99}
	 & \coloneqq
	Q_{0.99}\left(
	\frac{\norm{g_{\mathrm{train},s}}_2}{\norm{\theta_s}_2+\varepsilon}
	\right),                                                  \\
	\Phi_\tau
	 & \coloneqq
	\frac{1}{T_{\mathrm{step}}}\sum_{s=1}^{T_{\mathrm{step}}}
	\mathbf{1}\{\norm{g_{\mathrm{train},s}}_2>\tau\},         \\
	\Phi_{\mathrm{nf}}
	 & \coloneqq
	\frac{1}{T_{\mathrm{step}}}\sum_{s=1}^{T_{\mathrm{step}}}
	\mathbf{1}\{\mathrm{nonfinite}(g_{\mathrm{train},s})\}.
\end{align}
Here, $s$ indexes optimizer steps. $T_{\mathrm{step}}$ is the number of audited optimizer steps. In our stress runs, $\tau\in\{10,50,100\}$. When nonfinite gradients occur, we report both nonfinite-gradient fractions and finite-only gradient-tail summaries. This dual reporting prevents a single overflow event from obscuring the finite update-pressure distribution. For AdamW stress runs \citep{kingma2015adam,loshchilov2019decoupled}, we additionally sample the optimizer memory. Let $m_s$ and $v_s$ be the first and second moment tensors of Adam and let
\begin{align*}
	d_s^{\mathrm{proxy}}
	\coloneqq
	m_s / (\sqrt{v_s}+\epsilon_{\rm Adam})
\end{align*}
denote the uncorrected Adam preconditioner proxy before learning-rate and weight-decay scaling. The corresponding bias-corrected Adam preconditioner direction is
\begin{align*}
	d_s^{\mathrm{Adam}}
	\coloneqq
	\frac{m_s/(1-\beta_1^{t_s})}
	{\sqrt{v_s/(1-\beta_2^{t_s})}+\epsilon_{\rm Adam}},
\end{align*}
where $t_s$ is the internal Adam step count. We report the proxy because the uncorrected preconditioner tracks the same second-moment memory and coordinate-wise normalization while omitting bias-correction factors. The Adam $\epsilon_{\rm Adam}$ term is not bias-corrected in the same way as $v_s$. The proxy is therefore a diagnostic normalization rather than a constant multiple of the exact Adam direction. Specifically, we report $\max_s\norm{v_s}_\infty$, $Q_{0.99,s}\norm{v_s}_\infty$, $Q_{0.99,s}\norm{d_s^{\mathrm{proxy}}}_2/(\norm{\theta_s}_2+\varepsilon)$, and $Q_{0.99,s}\norm{d_s^{\mathrm{proxy}}}_\infty$. These endpoints distinguish smooth scalar-tail reduction from rare update events and optimizer-state memory that can dominate mixed-precision training.

Following standard mixed-precision concerns about fp16 dynamic range and loss scaling \citep{micikevicius2018mixed}, we also audit static loss-scale headroom. The audited quantity is the largest static fp16 loss scale that would keep already-computed optimizer-facing gradients representable if a stack stored scaled gradients in fp16. For an audited batch gradient vector $g_b$ and fp16 maximum $M_{16}\coloneqq65504$, define
\begin{align}
	\kappa_{\max,b}
	 & \coloneqq
	\frac{M_{16}}{\norm{g_b}_\infty+\varepsilon}, \\
	\Psi_\kappa
	 & \coloneqq
	\frac{1}{N_{\mathrm{batch}}}
	\sum_{b=1}^{N_{\mathrm{batch}}}
	\mathbf{1}\{\kappa\norm{g_b}_\infty>M_{16}\}.
	\label{eq:overflow_proxy}
\end{align}
This overflow endpoint targets static loss-scale headroom for stacks that store scaled gradients in fp16. As a stress endpoint for maximum-coordinate transport, lower maximum safe scales $\kappa_{\max,b}$ and higher $\Psi_\kappa$ indicate that rare coordinates leave less numerical headroom.

\subsubsection{Seed and Matched-Run Details}
\label{app:seed_details}

The matched comparisons use explicit run-count, seed, and token-accounting conventions. This matching design ensures that validation quality, transport diagnostics, and numerical endpoints are compared across the same randomized configurations. Unless otherwise stated, we use the same seed set for every method in grouped comparisons reported as mean$\pm$standard deviation over three matched runs. For tables that summarize replicated training endpoints without displaying standard deviations, we report the mean over the same matched seed set. Each paired difference is computed as target minus baseline after rows have been matched by seed and the rest of the run configuration. The matching fields include model, corpus, optimizer settings, data offsets, precision, clipping threshold, step count, and coefficient. Because lower values are preferred for PPL, CE, Diag. Z, tail, gain, and gradient-norm metrics, negative paired differences indicate improvement.

For dense WikiText-103 continued training, the GPT-2 comparison in Table~\ref{tab:finetune} uses three matched seeds. Rows report mean$\pm$standard deviation over the three matched runs. Each run starts from the same pretrained checkpoint, uses 4800 optimizer steps, consumes 39.3M training tokens, and is evaluated on 131,072 validation tokens. The GPT-2 diagnostic run in Appendix Table~\ref{tab:gpt2_diag} uses a separate three-seed matched set. The diagnostic run reports the three-run mean after 1200 steps, 9.83M training tokens per run, and 65,536 validation tokens. Table~\ref{tab:stress_diag} reports a high-coefficient GPT-2 gradient-tail stress sweep with three matched seeds. This sweep design gives 15 total runs across the five reported methods. Each run uses 1200 steps and 9.83M training tokens. Appendix Table~\ref{tab:medium_ft} reports the GPT-2 Medium scaling experiment with three matched seeds and 9.83M consumed tokens per run. Appendix Table~\ref{tab:medium_long} reports the GPT-2 Medium fixed-subset evaluation with three matched seeds and 19.7M consumed tokens per run. Appendix Table~\ref{tab:pythia_ft} reports the Pythia-160M continued pretraining experiment with three matched seeds and 4.9M consumed tokens per run. These three experiments report mean$\pm$standard deviation.

For FineWeb-Edu continued training, the GPT-2 experiment in Appendix Table~\ref{tab:fineweb_ft} uses three matched seeds. The GPT-2 experiment reports mean$\pm$standard deviation over three matched runs with 19.7M training tokens per run. Validation uses 131,072 disjoint next-token predictions from streamed blocks offset by 70,000 sequences from the training blocks. The GPT-2 Medium disjoint-block experiment uses a separate three-seed matched set. The fresh-token GPT-2 Medium experiment in Appendix Table~\ref{tab:fineweb_medium_ft} uses another three-seed matched set. The fresh-token experiment reports mean$\pm$standard deviation over three matched runs with 39.3M consumed tokens per run. The fresh-token training block set contains 41.0M available tokens, giving a token-reuse factor of 0.96. Validation uses a disjoint 220,000-sequence offset yielding 262,144 next-token predictions. Appendix Table~\ref{tab:fineweb_noclip_stress} reports the effectively unclipped FineWeb-Edu setting with three matched seeds. The effectively unclipped setting reports mean$\pm$standard deviation over three seeds. Each run uses 900 steps, a learning rate of $2\times 10^{-4}$, $\lambda_{\rm aux}=1\times 10^{-3}$ for auxiliary-method rows, a clipping threshold of $1\times 10^{8}$, and 3.69M consumed tokens from the 41.0M-token training block set. Each run uses the same disjoint 262,144-token validation blocks as Appendix Table~\ref{tab:fineweb_medium_ft}. The update-outlier experiment in Appendix Table~\ref{tab:fineweb_update_outlier} uses three matched seeds. The update-outlier experiment reports the three-seed mean under 600 steps, a learning rate of $8\times 10^{-4}$, $\lambda_{\rm aux}=3\times 10^{-3}$ for auxiliary-method rows, and a clipping threshold of $1\times 10^{8}$. Appendix Table~\ref{tab:fineweb_update_outlier_sweep} reports both the centered auxiliary-coefficient sweep and the matched raw-logit coefficient sweep under the same three-seed matched set, offsets, optimizer, and effectively unclipped setup. The Adam-state endpoint in Appendix Table~\ref{tab:adam_state_stress} uses a separate three-seed matched set. The Adam-state endpoint reports the three-seed mean over 300 steps, a learning rate of $8\times 10^{-4}$, $\lambda_{\rm aux}=3\times 10^{-3}$ for auxiliary-method rows, and AdamW state samples every 50 optimizer steps.

For sparse routing, the random router-projection audit uses one fixed seed. Router weights for each top-$k$ setting are generated by a deterministic initialization convention. Appendix Table~\ref{tab:moe} includes end-to-end top-$2$ and top-$4$ MoE experiments, each with three matched seeds. The top-$4$ rows report mean$\pm$standard deviation over three matched runs. Scale-matched active-route variants reuse the same seed and initialization as the corresponding token-layer-mean variants. This seed reuse is why the matched rows can recover the token-layer-mean result to numerical precision.

For deterministic audit rows, fixed seeds select evaluation blocks or randomized audit components before deterministic computation. We use one fixed seed per deterministic audit type, covering pretrained evaluation, source-space, \mbox{Triton}-based fused CE+Z-loss, optimizer-facing update, spectral, and static fp16 overflow audits.

\subsubsection{Additional Configuration Details}

The following details clarify the settings used for several Appendix tables and stress endpoints. Thresholded gradient-event rates are computed over optimizer steps, and paired comparisons are formed within matched experimental settings before target-minus-baseline differences are averaged across seeds. This matching avoids combining auxiliary configurations that share a model or method label but differ in an experiment-defining setting.

The spectral audit in Appendix Table~\ref{tab:spectral_audit} uses GPT-2 Medium on WikiText-103 with 128 sequences, an audit batch size of 2, and a rank of 16. The Pythia-1B diagnostic rows use the same audit settings with 64 sequences and an audit batch size of 2. We run the fresh-token GPT-2 Medium FineWeb-Edu experiment in Appendix Table~\ref{tab:fineweb_medium_ft} on four NVIDIA H100 accelerators using the FineWeb-Edu \texttt{sample-10BT} training split. The fresh-token experiment uses streamed text, the CE baseline, centered head, standard Z-loss, and centered Z-loss. The fresh-token experiment uses three matched seeds, 9600 steps, a micro-batch size of 8, an evaluation batch size of 8, a block size of 256, 160,000 training sequences, 1024 validation sequences, a training offset of 0, and a validation offset of 220,000. Optimization uses a learning rate of $5\times 10^{-5}$, a weight decay of 0.1, 480 warmup steps, gradient clipping at 1.0, a gradient-accumulation factor of 2, bf16 precision, an applied auxiliary coefficient of $1\times 10^{-4}$, and a gain coefficient of $1\times 10^{-2}$.

The stress experiment in Appendix Table~\ref{tab:fineweb_noclip_stress} uses the same training and validation offsets as Appendix Table~\ref{tab:fineweb_medium_ft}. The stress experiment evaluates GPT-2 Medium on the FineWeb-Edu \texttt{sample-10BT} training split with streamed text. In this experiment, we compare the CE baseline, centered head, standard Z-loss, factorized Z-loss, gain-aware Z-loss, and centered Z-loss. The experiment uses three matched seeds, 900 steps, a micro-batch size of 8, an evaluation batch size of 8, a block size of 256, 160,000 training sequences, and 1024 validation sequences. Optimization uses a learning rate of $2\times 10^{-4}$, a weight decay of 0.1, 90 warmup steps, a clipping threshold of $1\times 10^{8}$, a gradient-accumulation factor of 2, bf16 precision, an applied auxiliary coefficient of $1\times 10^{-3}$, and a gain coefficient of $1\times 10^{-2}$.

The update-outlier experiment in Appendix Table~\ref{tab:fineweb_update_outlier} uses the same model, dataset, streaming setup, offsets, and method family. The update-outlier experiment uses three matched seeds, 600 steps, a micro-batch size of 8, an evaluation batch size of 8, a block size of 256, 160,000 training sequences, and 512 validation sequences. Optimization uses a learning rate of $8\times 10^{-4}$, a weight decay of 0.1, 60 warmup steps, a clipping threshold of $1\times 10^{8}$, gradient-alert thresholds of 10, 50, and 100, a gradient-accumulation factor of 2, bf16 precision, an applied auxiliary coefficient of $3\times 10^{-3}$, and a gain coefficient of $1\times 10^{-2}$. The Adam-state endpoint in Appendix Table~\ref{tab:adam_state_stress} uses the same data offsets, optimizer, methods, and coefficient. The Adam-state endpoint uses a separate three-seed matched set, 300 steps, and optimizer state sampled every 50 optimizer steps.

For the static fp16 overflow audit in Appendix Table~\ref{tab:overflow_audit}, we use GPT-2 Medium on the FineWeb-Edu \texttt{sample-10BT} training split with streamed text. The audit uses 16 sequences, an audit batch size of 1, all-parameter gradients, and bf16 precision. The static fp16 audit compares the CE baseline, centered head, standard Z-loss, factorized Z-loss, gain-aware Z-loss, and centered Z-loss. Auxiliary-method rows use an applied coefficient of $3\times 10^{-3}$. The WikiText GPT-2 rows use the same audit design on WikiText-103 validation blocks with CE, standard Z-loss, and centered Z-loss.

For the centered auxiliary-coefficient sweep in Appendix Table~\ref{tab:fineweb_update_outlier_sweep}, we use the same three-seed matched set, offsets, optimizer, and stress setting while applying centered Z-loss with coefficients 0.001, 0.0003, and 0.0001. The matched raw-logit coefficient sweep uses the same stress setting for CE baseline, standard Z-loss, factorized Z-loss, and gain-aware Z-loss with coefficients 0.001, 0.0003, and 0.0001. CE baseline rows are grouped by coefficient for matched comparisons. The coefficient is not applied to CE-only training.

Interpretation of the reported tables therefore follows the experiment-specific settings above: paired results use matched seeds and consistent batch construction, precision, optimizer and clipping conventions, data offsets, and token budgets. The stress tables additionally summarize loss and gradient tails, while the Adam-state endpoint reports the corresponding sampled optimizer-state quantities.

\subsection{Additional Output-Head Geometry}
\label{app:additional_results}

\subsubsection{Tied-Embedding Pathway Diagnostics}
\label{app:tied_path_diagnostics}

The tied-embedding decomposition also shows pathway coupling. On GPT-2, the norm of the Z-loss gradient with respect to the tied embedding table is $3.5\%$ of the corresponding CE gradient norm for the diagnostic batch. The auxiliary-to-CE cosine is 0.022. Thus, the auxiliary update is nearly orthogonal to the CE update rather than simply reinforcing the likelihood gradient. The full Z-loss gradient decomposes into an output-path norm of 0.493 and an input-path norm of 0.563. The cosine between the input and output paths is $-0.0098$. On GPT-2 Medium, the same diagnostic gives $R_Z=3.5\%$, an auxiliary-to-CE cosine of 0.014, an output-path norm of 0.698, and an input-path norm of 0.375. Thus, the output auxiliary loss is not a single local regularizer on a tied head. The auxiliary loss creates a coupled multi-path update.

The scale diagnostics sharpen the claim. Pythia-160M provides a useful contrast because the Pythia-160M output projection is untied in this model family. The Pythia-160M diagnostic batch has a much larger Z-loss gradient budget, $R_Z=17.5\%$. The Pythia-160M auxiliary-to-CE cosine is also negative at $-0.149$. There is nevertheless no tied-table input-path contribution to the output-projection parameter. Z-loss can still propagate to the untied input embeddings and body through $W_U^\top\delta_z^Z$. The large raw common-shift tail in Table~\ref{tab:pretrained} persists beyond tied embeddings. Tying changes the transport path, while output-head common shift can be a separate and very large source factor. The GPT-2 Large rows in Table~\ref{tab:pretrained} show how the effect of centering depends on output-head geometry. The raw common shift already places $\log Z$ closer to the default target. Removing the raw common shift therefore preserves PPL but can increase Diag. Z and $A_p^{99}$. Pythia-410M and Pythia-1B provide the same cautionary contrast. These rows support the main position of this paper. Common-shift removal is a gauge intervention whose effect is evaluated with transport diagnostics.

\subsubsection{FineWeb-Edu Pretrained Diagnostics}

To test whether the WikiText results are corpus-specific, Appendix Table~\ref{tab:fineweb_pretrained} repeats the pretrained comparison of raw and centered output heads on streamed FineWeb-Edu blocks without continued training. Centering preserves PPL and removes the common-shift tail in all four model comparisons. Whether Diag. Z and the gain metrics increase or decrease still depends on the model geometry. GPT-2 and GPT-2 Medium repeat the large common-shift pattern from WikiText. GPT-2 Large and Pythia-410M show a geometry-dependent regime. Centering removes $\mu$ but increases $A_p^{99}$. For GPT-2 Large, centering also increases the $P_Z^{99.9}$ tail. This cross-corpus diagnostic replicates the common-shift pattern beyond WikiText. The cross-corpus diagnostic also shows that centering is best evaluated with transport diagnostics.

\begin{table}[t!]
	\centering
	\scriptsize
	\begin{tabular}{@{}lcrrrr@{}}
		\toprule
		Model                         & Head     & PPL   & Diag. Z  & $P_Z^{99.9}$ & $A_p^{99}$ \\
		\midrule
		\multirow{2}{*}{GPT-2}        & raw      & 31.65 & 1.260    & 261.9        & 29.35      \\
		                              & centered & 31.65 & 0.00680  & 22.5         & 16.92      \\
		\addlinespace
		\multirow{2}{*}{GPT-2 Medium} & raw      & 23.87 & 1.224    & 270.2        & 27.93      \\
		                              & centered & 23.87 & 0.00688  & 22.8         & 12.97      \\
		\addlinespace
		\multirow{2}{*}{GPT-2 Large}  & raw      & 20.08 & 0.000940 & 13.61        & 7.75       \\
		                              & centered & 20.08 & 0.00663  & 19.18        & 8.15       \\
		\midrule
		\multirow{2}{*}{Pythia-410M}  & raw      & 20.45 & 0.0162   & 20.54        & 2.74       \\
		                              & centered & 20.45 & 0.00719  & 18.40        & 4.04       \\
		\bottomrule
	\end{tabular}
	\caption{Pretrained diagnostics on FineWeb-Edu \texttt{sample-10BT}. GPT-2 and GPT-2 Medium mirror the WikiText common-shift pattern, while GPT-2 Large and Pythia-410M show a geometry-dependent regime where row-centered deployed logits increase gain-tail metrics.}
	\label{tab:fineweb_pretrained}
\end{table}
\FloatBarrier

\subsubsection{Unembedding Spectral Audit}

The gain identity and alignment remark yield a specific prediction. Large token-level gains $a_t$ should arise when the normalized softmax source aligns with unembedding directions associated with high singular values. The measurements in Appendix Table~\ref{tab:spectral_audit} numerically confirm the gain identity on corpus-derived logits. For GPT-2 and GPT-2 Medium, the top 16 singular directions contain between 33\% and 36\% of the raw unembedding Frobenius energy. The top 16 singular directions explain between 62\% and 64\% of token-level $a_t^2$ on average. The correlation between the rank-16 contribution and exact $a_t^2$ is at least 0.997. Centering removes a large mean-direction component. For GPT-2 Medium, the spectral norm drops from 462.9 to 87.9. $A_p^{99}$ drops from 30.5 to 14.5 on WikiText and from 28.0 to 13.1 on FineWeb-Edu. Pythia-160M is an extreme anisotropy case. The first raw singular direction alone contains 91.5\% of the unembedding energy. The top 16 directions explain 97.8\% of token-level $a_t^2$. After centering, the spectral norm drops from 776.4 to 106.1 and $A_p^{99}$ drops from 61.0 to 6.39. The complementary case appears in the FineWeb-Edu rows for GPT-2 Large and Pythia-410M and the WikiText row for Pythia-1B. Centering reduces the spectral norm but can increase $A_p^{99}$ when the softmax source aligns more closely with the remaining centered directions. This experiment gives a direct empirical bridge between \Eqref{eq:ap_svd} and the observed gradient-tail behavior.

\begin{table}[t!]
	\centering
	\scriptsize
	\resizebox{\columnwidth}{!}{%
		\begin{tabular}{@{}llcrrrrr@{}}
			\toprule
			Model                         & Corpus                       & Head     & $\sigma_1$ & $E_{16}$ & $A_p^{99}$ & $F_{16}$ & $\rho_{16}$ \\
			\midrule
			\multirow{2}{*}{GPT-2}        & \multirow{2}{*}{WikiText}    & raw      & 465.6      & 0.335    & 33.88      & 0.624    & 0.997       \\
			                              &                              & centered & 103.0      & 0.097    & 19.33      & 0.544    & 0.974       \\
			\addlinespace
			\multirow{2}{*}{GPT-2 Medium} & \multirow{2}{*}{WikiText}    & raw      & 462.9      & 0.362    & 30.54      & 0.644    & 0.998       \\
			                              &                              & centered & 87.9       & 0.085    & 14.49      & 0.448    & 0.957       \\
			\addlinespace
			\multirow{2}{*}{GPT-2 Medium} & \multirow{2}{*}{FineWeb-Edu} & raw      & 462.9      & 0.362    & 28.00      & 0.644    & 0.999       \\
			                              &                              & centered & 87.9       & 0.085    & 13.10      & 0.461    & 0.974       \\
			\addlinespace
			\multirow{2}{*}{GPT-2 Large}  & \multirow{2}{*}{FineWeb-Edu} & raw      & 147.8      & 0.178    & 7.75       & 0.254    & 0.963       \\
			                              &                              & centered & 67.1       & 0.103    & 8.15       & 0.506    & 0.977       \\
			\midrule
			\multirow{2}{*}{Pythia-160M}  & \multirow{2}{*}{WikiText}    & raw      & 776.4      & 0.931    & 60.96      & 0.978    & 1.000       \\
			                              &                              & centered & 106.1      & 0.317    & 6.39       & 0.492    & 0.987       \\
			\addlinespace
			\multirow{2}{*}{Pythia-410M}  & \multirow{2}{*}{FineWeb-Edu} & raw      & 55.8       & 0.137    & 2.74       & 0.428    & 0.922       \\
			                              &                              & centered & 29.6       & 0.069    & 4.04       & 0.599    & 0.988       \\
			\addlinespace
			\multirow{2}{*}{Pythia-1B}    & \multirow{2}{*}{WikiText}    & raw      & 64.4       & 0.106    & 2.73       & 0.281    & 0.675       \\
			                              &                              & centered & 31.0       & 0.045    & 4.39       & 0.507    & 0.973       \\
			\bottomrule
		\end{tabular}
	}%
	\caption{Low-rank unembedding spectral audit on corpus-derived text logits. $E_{16}$ is the fraction of unembedding Frobenius energy in the top 16 singular directions. $F_{16}$ is the mean token-level fraction of exact $a_t^2$ explained by those directions. $\rho_{16}$ is the Pearson correlation between exact $a_t^2$ and the rank-16 contribution.}
	\label{tab:spectral_audit}
\end{table}

\subsection{Why Precision Enters}
\label{sec:precision_link}

The connection between Z-loss and finite precision is two-sided. First, production, storage, or communication may round a raw logit before softmax evaluation with a numerically stable implementation. For a floating-point value in the normal range under round-to-nearest, the storage error obeys the standard model
\begin{equation}
	\begin{aligned}
		\left|\operatorname{fl}(z_i)-z_i\right|
		           & \le u_{\mathrm{mach}}|z_i|,       \\
		\max_i z_i & \le \log Z \le \max_i z_i+\log V,
	\end{aligned}
	\label{eq:precision_link}
\end{equation}
where $u_{\mathrm{mach}}$ is the unit roundoff \citep{goldberg1991floating}. A large common offset therefore coarsens the absolute spacing of representable logits and can erase small logit differences before the softmax kernel operates. Subsequent row-maximum subtraction prevents exponential overflow but cannot reconstruct differences already lost at an earlier low-precision boundary. The second inequality in \Eqref{eq:precision_link} shows why log-normalizer control is relevant. If $|\log Z-c|\le r$, then $c-r-\log V\le\max_i z_i\le c+r$. Keeping $\log Z$ near a finite target thus constrains the row maximum and common-shift excursions. Such log-normalizer control can preserve dense-head probabilities, router gate weights, and threshold-sensitive routing decisions when low-precision rounding precedes row-maximum subtraction \citep{zoph2022stmoe,blanchard2021accurate}. The benefit is implementation-dependent. Z-loss does not bound every logit or the full logit range, and Z-loss does not replace a stable log-sum-exp or softmax implementation.

Second, the Z-loss branch creates a separate backward precision obligation. A fused or low-precision path can match the scalar CE+Z value while reconstructing different $\log Z$ and $p$ values and therefore transporting a different $\delta_z^Z$. Once transported, a large source amplitude or architectural gain can create maximum-coordinate gradients that consume loss-scale headroom in stacks that store scaled gradients in fp16 or enter adaptive-optimizer moment state. Z-loss can therefore improve forward logit representability while worsening backward numerical headroom. Our audits distinguish these mechanisms. Source consistency, selected-gradient consistency, and optimizer-direction consistency are compared with an fp32 reference, while gradient tails and static fp16 margins measure the transported cost. The source-distortion identity below gives the exact source-mismatch decomposition and distinguishes Z-loss source consistency from standard row-maximum subtraction.

\subsection{Source Distortion and Finite-Precision Deployment}
\label{app:source_deployment}

\subsubsection{Source-Distortion Identity and Diagnostics}

If a fused or low-precision implementation produces $\widehat{\log Z}$ and $\widehat p$ during the backward pass, define $\epsilon_Z\coloneqq\widehat{\log Z}-\log Z$, $\epsilon_p\coloneqq\widehat p-p\in\R^V$, and $\widehat{\delta}_z^Z\coloneqq2\lambda(\widehat{\log Z}-c)\widehat p$. Then
\begin{equation}
	\widehat{\delta}_z^Z-\delta_z^Z
	=
	2\lambda\left(\epsilon_Zp+(\log Z-c)\epsilon_p+\epsilon_Z\epsilon_p\right).
	\label{eq:source_mismatch}
\end{equation}
Agreement between forward scalar loss values therefore does not imply agreement between the corresponding transported backward vectors. This distinction is especially important in low precision. Mathematically equivalent log-sum-exp and softmax formulas need not have identical floating-point behavior \citep{goldberg1991floating,blanchard2021accurate}. This source-distortion observation motivates separate source-space and transported-gradient consistency diagnostics,
\begin{equation}
	\begin{aligned}
		\Delta_{\mathrm{src}}
		 & \coloneqq
		\frac{\norm{\widehat{\delta}_z^Z-\delta_z^Z}_2}
		{\norm{\delta_z^Z}_2+\varepsilon},
		\\
		\Delta_{\theta}
		 & \coloneqq
		\frac{\norm{g_{Z,\theta}^{\mathrm{approx}}-g_{Z,\theta}^{\mathrm{ref}}}_F}
		{\norm{g_{Z,\theta}^{\mathrm{ref}}}_F+\varepsilon}.
	\end{aligned}
	\label{eq:backward_deltas}
\end{equation}
Row-maximum subtraction and Z-loss source consistency address complementary parts of a reliable CE+Z-loss implementation. Subtracting a row maximum is an algebraically invariant way to avoid overflow and reduce harmful underflow, and mature CE paths such as the built-in PyTorch kernels already use numerically stable softmax or log-softmax formulations \citep{paszke2019pytorch}. When a training stack adds Z-loss, especially in fused or low-precision CE+Z-loss code, the forward scalar and backward source for Z-loss must be derived from the same sufficient statistics rather than separately rounded or reconstructed quantities. A source-consistency audit therefore verifies that the auxiliary signal injected into backpropagation matches the implemented forward objective, complementing standard softmax stabilization in the kernel.

\subsubsection{Deployment of Centered Heads in Finite Precision}

Throughout, we use fp32, fp16, and bf16 as lowercase data type labels. In the experiments, centered logits are deployed by subtracting the fp32 output-head common shift from the already-computed raw logits, preserving the realized softmax distribution without a second low-precision output projection. \Eqref{eq:centered_raw} gives the bias-aware form and distinguishes this deployment from internal softmax row-maximum subtraction.

With low-precision weights, token-dependent residuals can arise when $\widetilde W_Uu$ is formed as a fresh fp16 or bf16 matrix product. Those residuals do not consist solely of a common shift. Our implementation therefore deploys centered logits by subtracting the output-head common shift from the already-computed raw logits:
\begin{equation}
	\begin{aligned}
		z_{\mathrm{raw}}
		 & \coloneqq
		W_Uu+b,
		\qquad
		\widetilde z\coloneqq z_{\mathrm{raw}}-\mu_{\mathrm{head}}\mathbf{1}, \\
		\mu_{\mathrm{head}}
		 & \coloneqq
		\frac{1}{V}\mathbf{1}^\top W_Uu+\bar b,
		\qquad
		\bar b\coloneqq\frac{1}{V}\sum_v b_v .
	\end{aligned}
	\label{eq:centered_raw}
\end{equation}
The formal propositions in the main text use the bias-free notation $\mu=\bar w u$. \Eqref{eq:centered_raw} is the corresponding deployment formula with an optional output bias. For a bias-free head, $\bar b=0$ and $\mu_{\mathrm{head}}=\mu$. In exact arithmetic, the raw-logit definition in \Eqref{eq:centered_raw} gives
\begin{align*}
	\frac{1}{V}\mathbf{1}^\top z_{\mathrm{raw}}
	=
	\mu_{\mathrm{head}},
\end{align*}
so \Eqref{eq:centered_raw} is equivalent to subtracting the realized raw-logit mean in exact arithmetic. In lower precision, the implementation preserves the realized softmax distribution by subtracting a token-wise common shift. The centered-logit deployment centers with respect to the fp32 output-head common shift, not necessarily the rounded raw-logit mean. This subtraction is also distinct from subtracting the per-row maximum inside a softmax kernel. Row-maximum subtraction is an internal numerical device. The per-row maximum is added back in a stable log-sum-exp computation, so row-maximum subtraction does not change the deployed objective coordinate. Row-centered deployment subtracts the output-head common-shift coordinate before the auxiliary loss is evaluated. Row-centered deployment preserves CE because CE is shift-invariant but intentionally changes the Z-loss source coordinate.

\subsection{Source-Consistency and Implementation Audits}
\label{app:implementation_audits}

For implementation sensitivity, we use a source-space audit with GPT-2 logits from WikiText-103 validation blocks. We quantize the logits to bf16 or fp16. We optionally accumulate the softmax statistics in fp32. We then compare the reconstructed Z-loss source to the fp32 reference. Appendix Table~\ref{tab:source_audit} shows that forward scalar agreement can substantially understate Z-loss source error. With bf16 logits at the natural scale, the mean relative forward Z-loss error is only $1.9\times 10^{-4}$. The mean Z-loss source relative error is 0.065. At a logit scale of 4, the source error rises to 0.183 while the forward relative error remains below $5\times 10^{-4}$. On GPT-2 logits from FineWeb-Edu, the same bf16 reconstruction gives forward relative errors below $5\times 10^{-4}$. The Z-loss source errors are 0.082, 0.143, and 0.207 at logit scales 1, 2, and 4. The fp16 source errors are smaller on both corpora, ranging from 0.008 to 0.027. This source-audit result illustrates that the relevant issue is not the data type alone but whether the backward path reconstructs the same sufficient statistics.

We observe similar behavior when applying the same source-space audit to GPT-2 Medium logits from WikiText and FineWeb-Edu. On WikiText, bf16 source reconstruction has a relative forward Z-loss error of $2.0\times 10^{-4}$ at the natural scale and $2.6\times 10^{-4}$ at a scale of 4. The corresponding Z-loss source relative errors are 0.057 and 0.171. On FineWeb-Edu, the corresponding forward errors are $1.7\times 10^{-4}$ and $2.3\times 10^{-4}$. The Z-loss source errors are 0.0687 and 0.193. The fp16 source errors are smaller, ranging from 0.007 to 0.026 across these GPT-2 Medium checks. The direct PyTorch CE backward audit also remains close to the reference for GPT-2 Medium. On FineWeb-Edu, the bf16 total-source relative differences are 0.00150, 0.00138, and 0.00108 at scales 1, 2, and 4.

We also run a direct PyTorch backward audit on the same corpus-derived logits. We compare built-in \texttt{cross\_entropy} plus manual Z-loss with a manual \texttt{log\_softmax} CE plus the same manual Z-loss in this audit. The discrepancies are much smaller than in the Z-loss source reconstruction stress test. In bf16, the relative total-loss backward difference is $1.40\times 10^{-3}$ at the natural scale and $1.04\times 10^{-3}$ at a scale of 4 on WikiText. On FineWeb-Edu, the corresponding differences are $1.44\times 10^{-3}$ and $1.05\times 10^{-3}$. In fp16, the difference is below $2.8\times 10^{-4}$. In fp32, the difference is zero within reporting precision. These results show that the built-in PyTorch CE path remains close to the fp32 reference in this setting. The same results also show the value of source-consistency checks for the actual fused CE+Z-loss implementation used by a training stack.

As a check of the fused implementation, we also implement a row-wise fused CE+Z-loss GPU kernel in \mbox{Triton}. The kernel computes the forward scalar, the total CE+Z logit gradient, and the Z-loss source from shared softmax statistics. Appendix Table~\ref{tab:triton_fused} reports this audit on eight GPT-2 corpus blocks from WikiText-103 and FineWeb-Edu. The fp32 kernel matches the PyTorch fp32 reference with a relative gradient error ranging from $1\times 10^{-7}$ to $1\times 10^{-6}$. This fp32 agreement confirms that the fused formula is source-consistent. When the same fused kernel reads bf16 logits, however, the forward Z-loss relative error stays below $3\times 10^{-4}$. The Z-loss source error is 0.056 on WikiText at the natural scale and 0.204 on FineWeb-Edu at a scale of 4. Thus, even a source-consistent fused kernel can transport a meaningfully different Z-loss signal if the sufficient statistics are reconstructed from quantized logits.

The fused audit with GPT-2 Medium extends the implementation observation beyond GPT-2 logits. The fp32 path again matches the PyTorch reference with a relative gradient error below $1\times 10^{-6}$. With bf16 logit storage, GPT-2 Medium has a Z-loss source relative error of 0.050 at the natural WikiText scale and 0.164 at a scale of 4. On FineWeb-Edu, the corresponding values are 0.0696 and 0.186. The total CE+Z logit-gradient error is smaller but still material, ranging from 0.033 to 0.134. The forward Z-loss relative error ranges only from $1.3\times 10^{-4}$ to $2.4\times 10^{-4}$. This GPT-2 Medium audit strengthens the conclusion that forward scalar agreement is an inadequate audit criterion even for a source-consistent fused formula.

\begin{table}[t!]
	\centering
	\scriptsize
	\begin{tabular}{@{}ccrrr@{}}
		\toprule
		Storage               & Logit $\times$ & \shortstack{Rel. forward                     \\error} & $\Delta_{\mathrm{src}}$ & \shortstack{Backward\\cosine} \\
		\midrule
		\multirow{3}{*}{bf16} & 1              & $1.91\times 10^{-4}$     & 0.0650  & 0.9978  \\
		                      & 2              & $2.43\times 10^{-4}$     & 0.1221  & 0.9925  \\
		                      & 4              & $2.67\times 10^{-4}$     & 0.1826  & 0.9832  \\
		\midrule
		\multirow{3}{*}{fp16} & 1              & $8.67\times 10^{-6}$     & 0.00816 & 0.99996 \\
		                      & 2              & $1.07\times 10^{-5}$     & 0.01548 & 0.99988 \\
		                      & 4              & $1.22\times 10^{-5}$     & 0.02408 & 0.99971 \\
		\bottomrule
	\end{tabular}
	\caption{Source precision audit on GPT-2 WikiText-103 logits. Forward scalar error can be tiny while the reconstructed backward source differs materially from the fp32 reference. Logit $\times$ is the multiplier applied to logits before the audit.}
	\label{tab:source_audit}
\end{table}

\begin{table}[t!]
	\centering
	\scriptsize
	\resizebox{\columnwidth}{!}{%
		\begin{tabular}{@{}llrrrr@{}}
			\toprule
			Corpus                       & \shortstack{Storage,                                                                                 \\logit $\times$} & \shortstack{Rel. Z forward\\error} & $\Delta_{\mathrm{tot}}$ & $\Delta_Z$ & \shortstack{Z\\cosine} \\
			\midrule
			\multirow{4}{*}{WikiText}    & fp32, 1              & $1.27\times 10^{-8}$ & $2.83\times 10^{-7}$ & $3.60\times 10^{-7}$ & 0.999998 \\
			                             & bf16, 1              & $2.48\times 10^{-4}$ & $3.43\times 10^{-2}$ & $5.62\times 10^{-2}$ & 0.9984   \\
			                             & bf16, 4              & $2.80\times 10^{-4}$ & $1.31\times 10^{-1}$ & $1.71\times 10^{-1}$ & 0.9851   \\
			                             & fp16, 4              & $2.30\times 10^{-5}$ & $1.64\times 10^{-2}$ & $2.10\times 10^{-2}$ & 0.9998   \\
			\midrule
			\multirow{3}{*}{FineWeb-Edu} & bf16, 1              & $1.95\times 10^{-4}$ & $3.82\times 10^{-2}$ & $8.62\times 10^{-2}$ & 0.9962   \\
			                             & bf16, 4              & $2.59\times 10^{-4}$ & $1.36\times 10^{-1}$ & $2.04\times 10^{-1}$ & 0.9791   \\
			                             & fp16, 4              & $2.68\times 10^{-5}$ & $1.78\times 10^{-2}$ & $2.69\times 10^{-2}$ & 0.9996   \\
			\bottomrule
		\end{tabular}
	}%
	\caption{Audit of a \mbox{Triton}-implemented fused CE+Z-loss GPU kernel on GPT-2 corpus logits. The kernel uses shared softmax statistics for forward and backward. $\Delta_{\mathrm{tot}}$ is the relative error in the total CE+Z logit gradient; $\Delta_Z$ is the relative error in the Z-loss source. Logit $\times$ is the pre-audit logit multiplier.}
	\label{tab:triton_fused}
\end{table}

\begin{table*}[t!]
	\centering
	\small
	\begin{tabular}{lclrrrr}
		\toprule
		Model                         & Scope                 & Corpus                       & Logit $\times$ & Grad. rel. error & Grad. cosine & Adam dir. rel. error \\
		\midrule
		\multirow{4}{*}{GPT-2}        & \multirow{4}{*}{head} & \multirow{2}{*}{WikiText}    & 1              & 0.0307           & 0.9995       & 0.139                \\
		                              &                       &                              & 4              & 0.1194           & 0.9929       & 0.247                \\
		                              &                       & \multirow{2}{*}{FineWeb-Edu} & 1              & 0.0368           & 0.9993       & 0.161                \\
		                              &                       &                              & 4              & 0.1236           & 0.9922       & 0.259                \\
		\addlinespace
		\multirow{4}{*}{GPT-2 Medium} & \multirow{4}{*}{head} & \multirow{2}{*}{WikiText}    & 1              & 0.0408           & 0.9992       & 0.128                \\
		                              &                       &                              & 4              & 0.1637           & 0.9866       & 0.228                \\
		                              &                       & \multirow{2}{*}{FineWeb-Edu} & 1              & 0.0397           & 0.9992       & 0.163                \\
		                              &                       &                              & 4              & 0.1453           & 0.9897       & 0.248                \\
		\midrule
		\multirow{2}{*}{GPT-2}        & \multirow{2}{*}{all}  & \multirow{2}{*}{WikiText}    & 1              & 0.0308           & 0.9995       & 0.144                \\
		                              &                       &                              & 4              & 0.0787           & 0.9968       & 0.192                \\
		\addlinespace
		\multirow{2}{*}{GPT-2 Medium} & \multirow{2}{*}{all}  & \multirow{2}{*}{WikiText}    & 1              & 0.0362           & 0.9993       & 0.146                \\
		                              &                       &                              & 4              & 0.1252           & 0.9921       & 0.181                \\
		\bottomrule
	\end{tabular}
	\caption{Optimizer-facing update audit. We backpropagate total CE+Z logit sources from the fp32 reference and the bf16 reconstruction through model graphs on corpus batches and compare output-head or whole-model gradients and first-step Adam directions. Whole-model rows use smaller WikiText batches because all transported gradients are materialized. Logit $\times$ is the pre-audit logit multiplier.}
	\label{tab:update_audit}
\end{table*}

Finally, we test whether the same source reconstruction error reaches optimizer-facing updates in addition to the logit-space source. We backpropagate either the fp32 reference total CE+Z logit source or a storage-quantized total CE+Z logit source through the same model graph on corpus batches. We then compare transported parameter gradients. The output-head block includes the direct output path. For tied heads in the GPT-2 family, the output-head block also includes the input-embedding path that returns through the Transformer body. Appendix Table~\ref{tab:update_audit} shows that the mismatch remains visible after transport. For bf16 storage, the relative difference in the total CE+Z output-head gradient ranges from 3.1\% to 4.1\% at the natural scale. At a logit scale of 4, the relative difference ranges from 11.9\% to 16.4\% across GPT-2 and GPT-2 Medium on WikiText and FineWeb-Edu. The corresponding first-step Adam direction mismatch is larger. The Adam direction mismatch ranges from 12.8\% to 16.3\% at the natural scale and from 22.8\% to 25.9\% at a scale of 4. Coordinate-wise adaptive normalization magnifies small sign and relative-magnitude changes. A smaller whole-model audit on WikiText gives comparable errors. At scales 1 and 4, the total-gradient relative errors are 3.1\% and 7.9\%, respectively, for GPT-2 and 3.6\% and 12.5\%, respectively, for GPT-2 Medium. The errors in the fp16 rows are much smaller. The fp16 total-gradient errors range from 0.4\% to 0.5\% at the natural scale and from 1.3\% to 2.0\% at a scale of 4 in the total-gradient audits. This update-level mismatch supports the practical audit recommendation. The backward source and update direction used by the actual training stack should be checked alongside the forward fused loss.

\FloatBarrier
\subsection{Additional Dense Continued-Pretraining Results}

\subsubsection{FineWeb-Edu GPT-2 Continued Pretraining}

Appendix Table~\ref{tab:fineweb_ft} reports a replication of the core GPT-2 comparison on FineWeb-Edu with disjoint training and validation blocks. This FineWeb-Edu comparison evaluates continued pretraining on web-corpus text rather than adaptation to WikiText validation. The CE baseline is again strong and low-variance. Centered head deployment preserves the mean PPL to within $1.6\times 10^{-4}$. Centered head deployment reduces the log-normalizer tail and the p99 output-to-hidden gain by factors of $9.6$ and $1.7$, respectively. Centered Z-loss has a mean PPL difference of $+0.0027$ relative to CE in this short comparison. Centered Z-loss gives the strongest joint tail reduction. The value of $P_Z^{99.9}$ falls from 246.5 to 19.5, and $A_p^{99}$ falls from 26.6 to 15.8. Standard Z-loss reduces the raw log-normalizer tail but does not reduce $A_p^{99}$. Standard Z-loss is also worse in validation PPL at this coefficient. The FineWeb-Edu result supports the main claim on a second corpus. Architecture-aware methods can reduce backward-geometry tails while maintaining comparable validation PPL.

\begin{table}[t!]
	\centering
	\scriptsize
	\resizebox{\columnwidth}{!}{%
		\begin{tabular}{@{}lrrrrrr@{}}
			\toprule
			Method          & PPL              & Diag. Z & $P_Z^{99.9}$ & $A_p^{99}$ & \shortstack{p99 pre-clip            \\$\|g\|_2$} & $C_{\rm aux}$ \\
			\midrule
			CE baseline     & $27.676\pm0.002$ & 0.1455  & 246.5        & 26.62      & 3.44                     & 0        \\
			Centered head   & $27.675\pm0.001$ & 0.00674 & 25.6         & 15.88      & 3.42                     & 0        \\
			\midrule
			Standard Z-loss & $27.699\pm0.002$ & 0.00418 & 33.4         & 26.62      & 3.62                     & $-0.088$ \\
			Centered Z-loss & $27.678\pm0.002$ & 0.00550 & 19.5         & 15.81      & 3.41                     & $-0.082$ \\
			\bottomrule
		\end{tabular}
	}%
	\caption{GPT-2 continued pretraining on FineWeb-Edu \texttt{sample-10BT} with disjoint streamed validation blocks.}
	\label{tab:fineweb_ft}
\end{table}
\FloatBarrier

\subsubsection{FineWeb-Edu GPT-2 Medium Fresh-Token Continued Pretraining}

Appendix Table~\ref{tab:fineweb_medium_ft} reports the GPT-2 Medium web-corpus evaluation with 39.3M consumed tokens and disjoint validation blocks. This experiment extends the FineWeb-Edu evaluation to GPT-2 Medium while avoiding repeated-subset ambiguity. As with the GPT-2 FineWeb-Edu experiment, validation-quality differences are small compared with the geometric effect. The mean paired PPL difference for centered head relative to CE is $-6.0\times 10^{-5}$, indicating essentially unchanged validation quality. Centered head deployment reduces $A_p^{99}$ by a factor of $2.19$. Standard Z-loss gives the largest reductions in scalar Z-loss and the p99 pre-clip full-model gradient norm. Standard Z-loss leaves $A_p^{99}$ at the CE level. Centered Z-loss gives the strongest joint scalar-tail and output-to-hidden gain reduction. Centered Z-loss reduces $P_Z^{99.9}$ and $A_p^{99}$ by factors of $7.5$ and $2.21$, respectively, with a measured PPL difference of $+0.025$ relative to the CE mean.

\begin{table}[t!]
	\centering
	\scriptsize
	\resizebox{\columnwidth}{!}{%
		\begin{tabular}{@{}lrrrrrr@{}}
			\toprule
			Method          & PPL              & Diag. Z & $P_Z^{99.9}$ & $A_p^{99}$ & \shortstack{p99 pre-clip            \\$\|g\|_2$} & $C_{\rm aux}$ \\
			\midrule
			CE baseline     & $21.597\pm0.008$ & 0.0347  & 127.4        & 23.59      & 3.88                     & 0        \\
			Centered head   & $21.597\pm0.008$ & 0.00487 & 20.3         & 10.79      & 3.88                     & 0        \\
			\midrule
			Standard Z-loss & $21.605\pm0.009$ & 0.00150 & 15.9         & 23.54      & 1.72                     & $-0.060$ \\
			Centered Z-loss & $21.622\pm0.009$ & 0.00407 & 16.9         & 10.65      & 3.84                     & 0.067    \\
			\bottomrule
		\end{tabular}
	}%
	\caption{GPT-2 Medium fresh-token continued pretraining on FineWeb-Edu \texttt{sample-10BT}.}
	\label{tab:fineweb_medium_ft}
\end{table}
\FloatBarrier

A paired analysis of Appendix Tables~\ref{tab:fineweb_ft} and~\ref{tab:fineweb_medium_ft} calibrates the same point. For GPT-2 on FineWeb-Edu, the mean paired PPL differences relative to CE are $-1.6\times 10^{-4}$ for centered head deployment and $+2.7\times 10^{-3}$ for centered Z-loss. Both reduce $A_p^{99}$ by about 10.8. In the fresh-token GPT-2 Medium grid on FineWeb-Edu, the corresponding mean paired PPL differences are $-6.0\times 10^{-5}$ for centered head deployment, $+2.52\times 10^{-2}$ for centered Z-loss, and $+8.64\times 10^{-3}$ for standard Z-loss. The mean paired differences in $A_p^{99}$ are $-12.80$ for centered head, $-12.94$ for centered Z-loss, and $-0.046$ for standard Z-loss. These differences confirm that raw-logit Z-loss and centered transport target different failure modes.

\subsubsection{1200-Step GPT-2 Diagnostic}

Appendix Table~\ref{tab:gpt2_diag} reports a 1200-step GPT-2 diagnostic run with post-training gradient instrumentation. This diagnostic run emphasizes gradient pathways and alignment rather than validation-quality rankings. The norm of the auxiliary gradient with respect to the tied embedding table ranges from $0.4\%$ to $0.5\%$ of the corresponding CE gradient norm. The auxiliary gradient is not zero and splits into output and input paths. Standard Z-loss leaves $A_p^{99}$ close to the CE baseline while reducing the common-shift tail. Centered Z-loss reduces both the log-normalizer tail and $A_p^{99}$. The auxiliary-to-CE cosine is slightly negative in this diagnostic batch. This negative cosine supports reporting alignment rather than assuming that the log-normalizer-regularization update is always CE-aligned.

\begin{table}[t!]
	\centering
	\scriptsize
	\resizebox{\columnwidth}{!}{%
		\begin{tabular}{@{}lrrrrrr@{}}
			\toprule
			Method          & PPL   & $P_Z^{99.9}$ & $A_p^{99}$ & \shortstack{p99 pre-clip                     \\$\|g\|_2$} & $R_{\rm aux}$ & $C_{\rm aux}$ \\
			\midrule
			CE baseline     & 20.97 & 256.9        & 27.33      & 6.94                     & 0      & 0        \\
			\midrule
			Standard Z-loss & 21.01 & 34.4         & 27.26      & 8.11                     & 0.0041 & $-0.056$ \\
			Centered Z-loss & 21.01 & 22.7         & 17.01      & 7.53                     & 0.0047 & $-0.057$ \\
			\bottomrule
		\end{tabular}
	}%
	\caption{GPT-2 1200-step diagnostic run on WikiText-103 with post-training gradient instrumentation.}
	\label{tab:gpt2_diag}
\end{table}
\FloatBarrier

\subsubsection{Model-Scale and Model-Family Transfer}

Appendix Table~\ref{tab:medium_ft} reports an extension of the comparison to GPT-2 Medium across three seeds. The CE baseline is strong and low-variance in this run. Standard Z-loss improves PPL and reduces $P_Z^{99.9}$. Standard Z-loss leaves the output-to-hidden gain tail near the CE baseline. Centered head deployment removes the common-shift coordinate and halves $A_p^{99}$. Centered head deployment targets the log-normalizer tail less directly. Centered Z-loss gives the best mean PPL in this grid. Centered Z-loss reduces $P_Z^{99.9}$ and $A_p^{99}$ by factors of $10.5$ and $1.9$, respectively, relative to the CE baseline. These results reinforce the interpretation that common-shift removal and output-to-hidden gain reduction are distinct interventions that combine constructively in this setting.

\begin{table}[t!]
	\centering
	\scriptsize
	\begin{tabular}{@{}lrrrr@{}}
		\toprule
		Method          & PPL              & Diag. Z & $P_Z^{99.9}$ & $A_p^{99}$ \\
		\midrule
		CE baseline     & $17.391\pm0.028$ & 0.2620  & 250.8        & 21.69      \\
		Centered head   & $17.393\pm0.027$ & 0.01183 & 32.9         & 11.46      \\
		\midrule
		Standard Z-loss & $17.358\pm0.029$ & 0.00283 & 27.4         & 21.61      \\
		Centered Z-loss & $17.326\pm0.027$ & 0.00700 & 23.9         & 11.26      \\
		\bottomrule
	\end{tabular}
	\caption{GPT-2 Medium scaling experiment on WikiText-103.}
	\label{tab:medium_ft}
\end{table}

Appendix Table~\ref{tab:medium_long} reports a 4800-step GPT-2 Medium evaluation on the same fixed training subset. The fixed-subset setup yields higher absolute PPL than the 2400-step scaling experiment. The higher absolute PPL is consistent with repeated exposure to a small subset rather than fresh large-corpus pretraining. Within this fixed-subset comparison, centered Z-loss is better than the CE baseline for every matched run. The mean paired PPL difference for centered Z-loss relative to CE is $-0.268$. Centered Z-loss reduces $P_Z^{99.9}$ and $A_p^{99}$ by factors of $9.7$ and $1.9$, respectively. Standard Z-loss also improves PPL and strongly reduces $P_Z^{99.9}$. Standard Z-loss leaves $A_p^{99}$ near the CE baseline. Centered head deployment reduces $A_p^{99}$ without the same targeted log-normalizer-tail reduction.

\begin{table}[t!]
	\centering
	\scriptsize
	\resizebox{\columnwidth}{!}{%
		\begin{tabular}{@{}lrrrrr@{}}
			\toprule
			Method          & \multicolumn{1}{c}{PPL} & Diag. Z & $P_Z^{99.9}$ & $A_p^{99}$ & \shortstack{p99 pre-clip \\$\|g\|_2$} \\
			\midrule
			CE baseline     & $20.328\pm0.012$        & 0.2402  & 247.4        & 19.20      & 4.42                     \\
			Centered head   & $20.329\pm0.005$        & 0.01582 & 40.7         & 10.33      & 4.43                     \\
			\midrule
			Standard Z-loss & $20.264\pm0.003$        & 0.00326 & 26.3         & 19.08      & 2.14                     \\
			Centered Z-loss & $20.060\pm0.015$        & 0.00798 & 25.6         & 9.98       & 4.40                     \\
			\bottomrule
		\end{tabular}
	}%
	\caption{GPT-2 Medium fixed-subset evaluation on WikiText-103. The fixed training subset separates this comparison from large-corpus pretraining results.}
	\label{tab:medium_long}
\end{table}

The experiment in Appendix Table~\ref{tab:pythia_ft} tests the same interventions on Pythia-160M. This model has a much larger raw common-shift channel than models in the GPT-2 family. The experiment is therefore a stress test for architecture-dependent transport. Centered head deployment reduces the CE-baseline $P_Z^{99.9}$ tail and $A_p^{99}$ by factors of $15.1$ and $8.7$, respectively. Centered head deployment also gives a mean PPL below that of CE, with larger run-to-run variation. Centered Z-loss gives the strongest joint reduction. Centered Z-loss reduces $P_Z^{99.9}$, $A_p^{99}$, and the p99 pre-clip full-model gradient norm by factors of $20.8$, $8.9$, and approximately $1.5$, respectively. Standard raw-logit Z-loss makes the scalar Z-loss almost vanish and reduces $P_Z^{99.9}$. Standard raw-logit Z-loss also worsens PPL by 9.2 points, leaves $A_p^{99}$ above the CE baseline, and raises the p99 pre-clip full-model gradient norm from 522 to 3612. Thus, nominally successful scalar Z-loss reduction can correspond to a worse transported update.

\begin{table}[t!]
	\centering
	\scriptsize
	\resizebox{\columnwidth}{!}{%
		\begin{tabular}{@{}lrrrrrr@{}}
			\toprule
			Method          & PPL            & Diag. Z & $P_Z^{99.9}$ & $A_p^{99}$ & \shortstack{p99 pre-clip          \\$\|g\|_2$} & $R_{\rm aux}$ \\
			\midrule
			CE baseline     & $35.97\pm3.79$ & 45.58   & 719.4        & 51.16      & 522                      & 0      \\
			Centered head   & $34.67\pm2.56$ & 0.0787  & 47.7         & 5.89       & 449                      & 0      \\
			\midrule
			Standard Z-loss & $45.17\pm1.43$ & 0.00584 & 38.3         & 55.46      & 3612                     & 0.0019 \\
			Centered Z-loss & $34.74\pm0.58$ & 0.0417  & 34.7         & 5.73       & 355                      & 0.0069 \\
			\bottomrule
		\end{tabular}
	}%
	\caption{Pythia-160M WikiText-103 continued pretraining. The Z-loss target in this grid is based on the actual output-head vocabulary size.}
	\label{tab:pythia_ft}
\end{table}

The continued pretraining result separates four effects. Standard Z-loss is a strong practical baseline for reducing log-normalizer tails. The PPL effects of standard Z-loss depend on the raw common-shift channel and training regime. Centered head deployment most directly removes the common-shift coordinate and lowers output-to-hidden gain tails. In the 4800-step GPT-2 run, centered head deployment matches CE within run-to-run variation. Centered Z-loss is the most balanced intervention in the high-coefficient GPT-2 setting and in the high-common-shift Pythia-160M setting. Centered Z-loss reduces both $P_Z^{99.9}$ and $A_p^{99}$ without the rare raw-logit gradient spikes seen in standard and gain-aware Z-loss. Factorized and gain-aware alternatives provide additional tuning knobs whose benefits depend on the target diagnostic. Together, these interventions form architecture-aware variants of Z-loss transport.

\subsection{High-Coefficient Stress and Precision Endpoints}
\label{app:stress_endpoints}

\subsubsection{Additional WikiText-103 Stress-Test Details}

The main-text stress test evaluates optimizer-facing consequences rather than introducing a new transport mechanism. The matched GPT-2 continued pretraining sweep in Table~\ref{tab:stress_diag} increases the applied auxiliary coefficient tenfold to $\lambda_{\rm aux}=1\times 10^{-3}$. Centered Z-loss is the only auxiliary method that improves the 1200-step mean PPL relative to CE while also substantially reducing both $P_Z^{99.9}$ and $A_p^{99}$ and keeping the p99 pre-clip full-model gradient norm near the CE level. Gain-aware Z-loss also improves mean PPL and reduces the log-normalizer tail. Because gain-aware Z-loss operates on raw logits, the gain-aware objective leaves $A_p^{99}$ near the CE baseline and shows rare large gradient-norm spikes. Standard Z-loss has the same raw-logit geometry issue and a small PPL increase relative to CE in this high-coefficient setting. The Centered head row is the gauge-only control. The lower $A_p^{99}$ of the Centered head row exposes the centered transport geometry, while the row has no auxiliary branch, so the reported $P_Z^{99.9}$ is a deployed-coordinate diagnostic rather than the outcome of Z-loss training.

Clip rate and p99 gradient norm summarize different parts of the training-step distribution. CE, Centered head, and Centered Z-loss exceed the 1.0 clipping threshold on nearly every step but remain near the CE-scale p99 norm. Standard and Gain-aware Z-loss exceed the threshold on only about $7\%$ of steps, yet the p99 norms of Standard and Gain-aware Z-loss are much larger. The lower clip rate therefore does not indicate a safer tail. The lower rate reflects mostly sub-threshold steps punctuated by rare, very large events. This distributional distinction is the reason we report both threshold frequencies and tail magnitudes.

Appendix Tables~\ref{tab:fineweb_noclip_stress}, \ref{tab:fineweb_update_outlier}, \ref{tab:fineweb_update_outlier_sweep}, \ref{tab:adam_state_stress}, and~\ref{tab:overflow_audit} report further stress and precision endpoints, including effectively unclipped stress tests, update-outlier and coefficient sweeps, Adam-state audits, and static fp16 overflow audits. These endpoints extend Table~\ref{tab:stress_diag}. Raw-logit scalar-tail reduction can amplify transported update tails. Centered transport reduces output-to-hidden gain and update-tail metrics with regime-dependent quality trade-offs.

These endpoints measure update pressure and numerical headroom rather than observed training divergence. No nonfinite gradients occurred in the reported FineWeb-Edu stress runs. The static fp16 endpoint is a held-out proxy. The endpoint applies candidate static loss scales to gradients computed in the stated audit under bf16 autocast and asks whether any scaled coordinate would exceed the fp16 range. The static fp16 endpoint therefore quantifies margin under a specified deployment scenario rather than claiming that the corresponding training run actually overflowed.

Appendix Tables~\ref{tab:medium_ft}, \ref{tab:medium_long}, and~\ref{tab:pythia_ft} report model-scale and model-family transfer experiments with GPT-2 Medium and Pythia-160M. These results do not introduce a separate mechanism. The transfer experiments test the same dense-head interpretation across scale and model family. Common-shift removal and output-to-hidden gain reduction are distinct interventions. Scalar Z-loss tail reduction can differ from transported-update tail reduction.

\FloatBarrier
\subsubsection{Effectively Unclipped and Update-Outlier Stress}

Appendix Table~\ref{tab:fineweb_noclip_stress} reports results from a 900-step effectively unclipped update-tail evaluation using GPT-2 Medium on FineWeb-Edu. The 900-step setting uses the same disjoint validation offset as Appendix Table~\ref{tab:fineweb_medium_ft}. The setting is more aggressive than the low-coefficient FineWeb-Edu comparison. The 900-step setting uses a learning rate of $2\times 10^{-4}$, $\lambda_{\rm aux}=1\times 10^{-3}$, a clipping threshold of $1\times 10^{8}$, and 900 optimizer steps. The result again separates scalar-tail reduction from transported-update-tail reduction. Standard raw-logit Z-loss gives the smallest scalar tail, reducing $P_Z^{99.9}$ by a factor of $16.5$. Under standard Z-loss, the p99 pre-clip full-model gradient norm rises by a factor of $4.1$. Standard Z-loss also raises the maximum observed pre-clip full-model gradient norm by a factor of $7.8$. Factorized and gain-aware variants confirm that the issue is not unique to the standard scalar objective. Both reduce the scalar tail. Because the factorized and gain-aware variants still leave the raw-logit output-to-hidden gain near CE, the p99 pre-clip full-model gradient norms of these two variants rise to 23.0 and 13.0. Centered Z-loss prioritizes tail reduction at a measured PPL cost in this 900-step setting. Centered Z-loss lowers both $P_Z^{99.9}$ and $A_p^{99}$ while leaving p99 and maximum pre-clip full-model gradient norms at the CE scale. Centered head deployment is quality-neutral in this setting and halves $A_p^{99}$ without introducing update outliers. No run produced gradients containing not-a-number (NaN) or infinity (Inf) values.

\begin{table}[t!]
	\centering
	\scriptsize
	\resizebox{\columnwidth}{!}{%
		\begin{tabular}{@{}lrrrrr@{}}
			\toprule
			Method            & PPL              & $P_Z^{99.9}$ & $A_p^{99}$ & \shortstack{p99 pre-clip         \\$\|g\|_2$} & \shortstack{max pre-clip\\$\|g\|_2$} \\
			\midrule
			CE baseline       & $22.197\pm0.009$ & 161.2        & 24.68      & 4.61                     & 5.61  \\
			Centered head     & $22.196\pm0.005$ & 19.8         & 11.59      & 4.59                     & 5.62  \\
			\midrule
			Standard Z-loss   & $22.214\pm0.005$ & 9.75         & 24.95      & 18.94                    & 43.93 \\
			Factorized Z-loss & $22.409\pm0.003$ & 16.1         & 24.81      & 23.03                    & 53.80 \\
			Gain-aware Z-loss & $22.207\pm0.006$ & 10.9         & 24.86      & 13.01                    & 47.24 \\
			\midrule
			Centered Z-loss   & $22.348\pm0.006$ & 13.4         & 11.13      & 4.58                     & 5.70  \\
			\bottomrule
		\end{tabular}
	}%
	\caption{Effectively unclipped high-coefficient GPT-2 Medium stress on FineWeb-Edu.}
	\label{tab:fineweb_noclip_stress}
\end{table}

The experiment in Appendix Table~\ref{tab:fineweb_update_outlier} increases both the learning rate and the auxiliary coefficient to measure optimizer-facing outliers. PPL and transport diagnostics rank the methods differently in this setting. At $\lambda_{\rm aux}=3\times 10^{-3}$, standard, factorized, and gain-aware raw-logit variants all improve 600-step PPL relative to CE while lowering the scalar log-normalizer tail. The backward transport diagnostics reveal the cost. Standard Z-loss raises the p99 and maximum pre-clip full-model gradient norms by factors of $5.3$ and $22.3$, respectively. Factorized Z-loss raises the p99 and maximum norms by factors of $5.5$ and $29.8$, respectively. Gain-aware Z-loss is milder but still raises the p99 and maximum pre-clip full-model gradient norms by factors of $3.3$ and $17.3$, respectively. These high-coefficient raw-logit variants have $\|g\|>10$ events in 1.56\% to 1.72\% of optimizer steps. The high-coefficient raw-logit variants have $\|g\|>100$ events in 0.11\% to 0.61\% of optimizer steps.

The coefficient sweeps in Appendix Table~\ref{tab:fineweb_update_outlier_sweep} map this trade-off. Raw-logit variants reduce thresholded-update rates as the coefficient decreases. At $\lambda_{\rm aux}=1\times 10^{-4}$, standard and gain-aware Z-loss are near CE in PPL and have zero $\|g\|>10$, $\|g\|>50$, and $\|g\|>100$ events. Standard and gain-aware Z-loss still leave $A_p^{99}$ at the CE scale, about 25.0. The scalar-tail reduction from these raw-logit variants is therefore not output-to-hidden gain reduction. Centered Z-loss shows the complementary trade-off. At $\lambda_{\rm aux}=3\times 10^{-3}$, centered Z-loss prioritizes update-event suppression. Centered Z-loss keeps all update-event rates at zero, with a PPL 0.60 higher than the CE baseline. Lowering the coefficient to $1\times 10^{-4}$ recovers near-CE PPL, with 27.774 for centered Z-loss and 27.766 for CE. Centered Z-loss still reduces $P_Z^{99.9}$ from 47.3 to 15.8 and $A_p^{99}$ from 25.0 to 11.7. Centered Z-loss also keeps all $\|g\|>10$, $\|g\|>50$, and $\|g\|>100$ event rates at zero.

\begin{table*}[t!]
	\centering
	\small
	\begin{tabular}{lrrrrrrrr}
		\toprule
		Method            & PPL    & $P_Z^{99.9}$ & $A_p^{99}$ & p99 pre-clip $\|g\|_2$ & $\Phi_{10}$ & $\Phi_{50}$ & $\Phi_{100}$ & max pre-clip $\|g\|_2$ \\
		\midrule
		CE baseline       & 27.766 & 47.3         & 25.01      & 4.46                   & 0.0000      & 0.0000      & 0.0000       & 5.86                   \\
		Centered head     & 27.777 & 16.9         & 11.88      & 4.48                   & 0.0000      & 0.0000      & 0.0000       & 5.86                   \\
		\midrule
		Standard Z-loss   & 26.945 & 7.56         & 25.17      & 23.48                  & 0.0172      & 0.0067      & 0.0061       & 129.8                  \\
		Factorized Z-loss & 27.479 & 11.71        & 24.93      & 24.38                  & 0.0172      & 0.0072      & 0.0050       & 174.9                  \\
		Gain-aware Z-loss & 27.081 & 7.75         & 25.09      & 14.89                  & 0.0156      & 0.0067      & 0.0011       & 101.2                  \\
		\bottomrule
	\end{tabular}
	\caption{Stronger GPT-2 Medium update-outlier stress on FineWeb-Edu at $\lambda_{\rm aux}=3\times 10^{-3}$ for auxiliary-method rows. Columns $\Phi_{10}$, $\Phi_{50}$, and $\Phi_{100}$ are pre-clip full-model gradient-norm event fractions, with $\Phi_\tau\coloneqq\Pr_{\rm step}[\|g\|_2>\tau]$. CE and centered head rows have no applied auxiliary loss. No microbatches or gradients contained NaN or Inf values in any run.}
	\label{tab:fineweb_update_outlier}
\end{table*}

\begin{table*}[t!]
	\centering
	\scriptsize
	\begin{tabular}{lrrrrrrrrr}
		\toprule
		Method            & $\lambda_{\rm aux}$ & PPL    & $P_Z^{99.9}$ & $A_p^{99}$ & p99 pre-clip $\|g\|_2$ & $\Phi_{10}$ & $\Phi_{50}$ & $\Phi_{100}$ & max pre-clip $\|g\|_2$ \\
		\midrule
		\multicolumn{10}{l}{Raw-logit variants}                                                                                                                                   \\
		Standard Z-loss   & $3\times 10^{-4}$   & 27.659 & 10.97        & 25.02      & 5.23                   & 0.0050      & 0.0000      & 0.0000       & 13.7                   \\
		Factorized Z-loss & $3\times 10^{-4}$   & 27.773 & 16.04        & 24.62      & 6.17                   & 0.0050      & 0.0000      & 0.0000       & 15.0                   \\
		Gain-aware Z-loss & $3\times 10^{-4}$   & 27.666 & 11.85        & 25.02      & 4.23                   & 0.0022      & 0.0000      & 0.0000       & 10.7                   \\
		Standard Z-loss   & $1\times 10^{-4}$   & 27.752 & 14.30        & 25.13      & 2.64                   & 0.0000      & 0.0000      & 0.0000       & 6.83                   \\
		Factorized Z-loss & $1\times 10^{-4}$   & 27.819 & 20.87        & 24.73      & 3.22                   & 0.0000      & 0.0000      & 0.0000       & 7.18                   \\
		Gain-aware Z-loss & $1\times 10^{-4}$   & 27.753 & 14.95        & 25.00      & 2.31                   & 0.0000      & 0.0000      & 0.0000       & 6.37                   \\
		\midrule
		\multicolumn{10}{l}{Centered Z-loss}                                                                                                                                      \\
		Centered Z-loss   & $3\times 10^{-3}$   & 28.369 & 11.15        & 10.81      & 4.56                   & 0.0000      & 0.0000      & 0.0000       & 6.09                   \\
		Centered Z-loss   & $1\times 10^{-3}$   & 27.987 & 13.22        & 11.13      & 4.38                   & 0.0000      & 0.0000      & 0.0000       & 5.96                   \\
		Centered Z-loss   & $3\times 10^{-4}$   & 27.840 & 15.00        & 11.53      & 4.50                   & 0.0000      & 0.0000      & 0.0000       & 5.86                   \\
		Centered Z-loss   & $1\times 10^{-4}$   & 27.774 & 15.77        & 11.68      & 4.54                   & 0.0000      & 0.0000      & 0.0000       & 5.83                   \\
		\bottomrule
	\end{tabular}
	\caption{FineWeb-Edu update-outlier coefficient sweeps. Columns $\Phi_{10}$, $\Phi_{50}$, and $\Phi_{100}$ are pre-clip full-model gradient-norm event fractions, with $\Phi_\tau\coloneqq\Pr_{\rm step}[\|g\|_2>\tau]$.}
	\label{tab:fineweb_update_outlier_sweep}
\end{table*}

\subsubsection{Optimizer-State and Overflow Endpoints}

Appendix Table~\ref{tab:adam_state_stress} reports the optimizer-memory endpoint under the same effectively unclipped FineWeb-Edu stress setting. The Adam-state endpoint samples AdamW state every 50 steps over 300 optimizer steps. Under this high-threshold clipping convention, these measurements show how raw-logit gradient spikes affect an adaptive optimizer. Standard, factorized, and gain-aware Z-loss improve 300-step PPL and reduce $P_Z^{99.9}$. The associated raw-logit gradient spikes are nevertheless written into the second-moment accumulator. Relative to CE, standard, factorized, and gain-aware Z-loss increase the maximum sampled Adam second-moment accumulator $v$ by factors of $36.7$, $42.1$, and $20.3$, respectively. The full preconditioned Adam direction norm is partly damped by this larger $v$, as expected from coordinate-wise normalization in Adam by $\sqrt{v}$. The maximum-coordinate Adam direction still increases slightly. Centered head deployment and centered Z-loss keep the Adam $v$ tail close to CE. Centered head deployment and centered Z-loss also preserve a zero pre-clip gradient-event rate. Centered Z-loss shows the corresponding PPL trade-off at this aggressive coefficient. This Adam-state endpoint links the tail diagnostic to optimizer memory and update coordinates, while clip-before-Adam training stacks can reduce this particular memory effect by clipping spikes before the Adam update.

\begin{table*}[t!]
	\centering
	\small
	\begin{tabular}{lrrrrrrr}
		\toprule
		Method            & PPL    & $P_Z^{99.9}$ & $A_p^{99}$ & p99 pre-clip $\|g\|_2$ & $\Phi_{100}$ & $v_{\max}$ ratio to CE & p99 $\|d\|_\infty$ \\
		\midrule
		CE baseline       & 26.899 & 69.43        & 25.58      & 4.98                   & 0.0000       & 1.00                   & 1.106              \\
		Centered head     & 26.897 & 18.13        & 12.07      & 4.94                   & 0.0000       & 0.98                   & 1.108              \\
		\midrule
		Standard Z-loss   & 25.406 & 9.01         & 25.54      & 81.23                  & 0.0100       & 36.7                   & 1.126              \\
		Factorized Z-loss & 25.786 & 13.81        & 25.68      & 95.72                  & 0.0111       & 42.1                   & 1.122              \\
		Gain-aware Z-loss & 25.599 & 9.33         & 25.36      & 59.47                  & 0.0033       & 20.3                   & 1.123              \\
		\midrule
		Centered Z-loss   & 27.361 & 11.45        & 11.24      & 5.12                   & 0.0000       & 1.12                   & 1.102              \\
		\bottomrule
	\end{tabular}
	\caption{Adam-state endpoint for effectively unclipped GPT-2 Medium stress on FineWeb-Edu. $v_{\max}$ is the maximum sampled Adam second moment; the $v_{\max}$ ratio to CE is computed relative to the matched CE baseline. All recorded gradients and sampled optimizer-state values remained finite.}
	\label{tab:adam_state_stress}
\end{table*}

Appendix Table~\ref{tab:overflow_audit} gives a second numerical endpoint that is independent of the Adam state. We backpropagate the same high-coefficient objectives on held-out corpus text. We inspect full-model gradients before any optimizer step and measure static fp16 loss-scale headroom. The endpoint complements the Adam-state analysis by measuring maximum-coordinate headroom. The ordering is consistent with the effectively unclipped and Adam-state results. For GPT-2 on WikiText-103, standard raw-logit Z-loss increases the mean full-model gradient norm from 17.0 to 191.2. Standard raw-logit Z-loss increases the mean maximum coordinate from 5.77 to 70.0 and reduces the worst-case maximum safe static scale from 6718 to 731. At a static scale of 4096, standard raw-logit Z-loss gradients overflow in every audited batch, while CE and centered Z-loss do not. For GPT-2 Medium on FineWeb-Edu, standard, factorized, and gain-aware raw-logit variants reduce the worst-case maximum safe scale to 1483, 1437, and 2248, respectively. The raw-logit variants produce overflow at a static scale of 4096 in 87.5\% to 100\% of audited batches. Gradients from the centered head and centered Z-loss stay at the CE scale. Centered-head and centered-Z-loss gradients have worst-case maximum safe scales around 8900, and neither centered method produces overflow at a static scale of 4096. This overflow endpoint links the transported maximum-coordinate tail to a concrete mixed-precision numerical margin.

\begin{table*}[t!]
	\centering
	\scriptsize
	\begin{tabular}{llrrrrr}
		\toprule
		Setting                                   & Method            & mean $\norm{g}_2$ & mean $\norm{g}_\infty$ & $\kappa_{\min}$ & $\Psi_{4096}$ & $\Psi_{16384}$ \\
		\midrule
		\multirow{3}{*}{WikiText GPT-2}           & CE baseline       & 17.0              & 5.77                   & 6718            & 0.000         & 0.938          \\
		                                          & Standard Z-loss   & 191.2             & 70.0                   & 731             & 1.000         & 1.000          \\
		                                          & Centered Z-loss   & 17.0              & 5.66                   & 6805            & 0.000         & 0.875          \\
		\midrule
		\multirow{6}{*}{FineWeb-Edu GPT-2 Medium} & CE baseline       & 18.0              & 4.63                   & 8853            & 0.000         & 0.625          \\
		                                          & Centered head     & 18.0              & 4.63                   & 8853            & 0.000         & 0.625          \\
		                                          & Standard Z-loss   & 188.3             & 30.8                   & 1483            & 1.000         & 1.000          \\
		                                          & Factorized Z-loss & 199.0             & 32.4                   & 1437            & 1.000         & 1.000          \\
		                                          & Gain-aware Z-loss & 143.4             & 22.4                   & 2248            & 0.875         & 1.000          \\
		                                          & Centered Z-loss   & 17.9              & 4.53                   & 8928            & 0.000         & 0.625          \\
		\bottomrule
	\end{tabular}
	\caption{Static fp16 loss-scale overflow proxy on full-model gradients from held-out corpus blocks. All rows use bf16 autocast and 16 held-out sequence blocks; auxiliary-method rows use $\lambda_{\rm aux}=3\times 10^{-3}$. Standard, Factorized, and Gain-aware Z-loss use raw logits. Centered head has no auxiliary loss, while Centered Z-loss uses row-centered deployed logits. For Factorized Z-loss, $\lambda_\mu=\lambda_{\mathrm{rel}}=\lambda_{\rm aux}$, $c_\mu=0$, and $c_{\mathrm{rel}}=\log V$. Gain-aware Z-loss uses $\beta=1\times 10^{-2}$. $\kappa_{\min}$ is the minimum over audited batches of the per-batch maximum safe scale $65504/\norm{g}_\infty$. Each of $\Psi_{4096}$ and $\Psi_{16384}$ denotes the fraction of audited batches whose scaled gradient would exceed the fp16 range at the corresponding static scale, so each lies in $[0,1]$.}
	\label{tab:overflow_audit}
\end{table*}

\FloatBarrier
\onecolumn
\twocolumn[
	\subsection{Additional Router and MoE Training Results}
	\label{app:moe_training_results}
]
\raggedbottom

\subsubsection{Router Projection Audit}

We compare token-layer mean, active-route mean, and active-route sum reductions for a 16-expert linear router applied to GPT-2 hidden states from WikiText-103; $\delta_h^Z$ denotes the router Z-loss correction transported to the router input.

\subsubsection{End-to-End MoE Training}

We then train controlled MoE language models end-to-end on WikiText-103. Appendix Table~\ref{tab:moe} reports the top-$4$ setting, where the scale gap between active-route mean and active-route sum is largest. The matched variants use adjusted coefficients so that the scale of each variant relative to the token-layer mean equals 1. Consistent with the router effective-scale accounting above, the matched variants reproduce the token-layer-mean result to numerical precision within each matched run. This matched recovery is an end-to-end equivalence check. Under the reused initialization, coefficient matching makes the effective router objectives equivalent, and their matched recovery confirms the predicted reduction-scale accounting at the end-to-end training level. Unmatched active-route mean uses one quarter of the coefficient required to match the token-layer mean and has lower router entropy. Active-route sum uses four times the coefficient required to match the token-layer mean and has higher router entropy. The 1200-step setting yields similar validation CE across rows. The router metrics follow the predicted scale accounting. This controlled MoE training experiment shows that absolute per-decision coefficients and scales relative to the token-layer mean make reduction conventions comparable. The top-$2$ setting shows the same matched-variant recovery with relative scales 0.5, 1.0, and 2.0.

\noindent\begin{minipage}{\columnwidth}
	\centering
	\scriptsize
	\begin{tabular}{@{}clrr@{}}
		\toprule
		Top-$k$            & Convention        & \shortstack{Scale relative to                        \\token-layer mean} & \shortstack{Mean router\\$\norm{\delta_h^Z}_2$} \\
		\midrule
		\multirow{3}{*}{1} & token-layer mean  & 1.00                          & $9.41\times 10^{-9}$ \\
		                   & active-route mean & 1.00                          & $9.41\times 10^{-9}$ \\
		                   & active-route sum  & 1.00                          & $9.41\times 10^{-9}$ \\
		\midrule
		\multirow{3}{*}{2} & token-layer mean  & 1.00                          & $1.59\times 10^{-8}$ \\
		                   & active-route mean & 0.50                          & $7.96\times 10^{-9}$ \\
		                   & active-route sum  & 2.00                          & $3.18\times 10^{-8}$ \\
		\midrule
		\multirow{3}{*}{4} & token-layer mean  & 1.00                          & $4.59\times 10^{-9}$ \\
		                   & active-route mean & 0.25                          & $1.15\times 10^{-9}$ \\
		                   & active-route sum  & 4.00                          & $1.83\times 10^{-8}$ \\
		\bottomrule
	\end{tabular}
	\captionof{table}{Router Z-loss reduction conventions for a 16-expert linear router applied to GPT-2 hidden states obtained from WikiText-103 validation blocks. Effective scale and hidden correction norm change with $k$ despite the same nominal $\lambda_R=1\times 10^{-3}$.}
	\label{tab:router}
\end{minipage}

\noindent\begin{minipage}{\columnwidth}
	\centering
	\scriptsize
	\resizebox{\columnwidth}{!}{%
		\begin{tabular}{@{}lrrrrr@{}}
			\toprule
			Router loss               & $s_R^{\mathrm{rel}}$ & CE                & PPL           & Entropy & Load CV \\
			\midrule
			none                      & 0.0                  & $5.8361\pm0.0095$ & $342.5\pm3.3$ & 1.714   & 0.192   \\
			\midrule
			token-layer mean          & 1.0                  & $5.8352\pm0.0090$ & $342.2\pm3.1$ & 1.725   & 0.187   \\
			\addlinespace
			active-route mean         & 0.25                 & $5.8355\pm0.0093$ & $342.2\pm3.2$ & 1.713   & 0.191   \\
			active-route mean matched & 1.0                  & $5.8352\pm0.0090$ & $342.2\pm3.1$ & 1.725   & 0.187   \\
			\addlinespace
			active-route sum          & 4.0                  & $5.8358\pm0.0095$ & $342.4\pm3.2$ & 1.756   & 0.190   \\
			active-route sum matched  & 1.0                  & $5.8352\pm0.0090$ & $342.2\pm3.1$ & 1.725   & 0.187   \\
			\bottomrule
		\end{tabular}
	}%
	\captionof{table}{End-to-end top-$4$ MoE language-model training on WikiText-103. Scale-matched active-route variants recover the token-layer-mean baseline; unmatched variants move router entropy according to each variant's effective scale. Load coefficient of variation (CV) reports expert-load dispersion.}
	\label{tab:moe}
\end{minipage}

\subsection{Reporting Recommendations and Limitations}

\subsubsection{Z-Loss Reporting Recommendations}

The experiments yield a practical reporting standard for Z-loss. For dense output heads, comparable reports should include the applied auxiliary coefficient $\lambda_{\rm aux}$ and the diagnostic coefficient $\lambda_{\rm diag}$ when different. Comparable reports should also include the target $c$, the logit convention, whether embeddings are tied, and whether the auxiliary is applied to raw or deployed logits. Reports should specify whether the logits are raw, row-centered deployed logits, or otherwise normalized. Scalar validation quality is most interpretable alongside transport diagnostics. These transport diagnostics include $P_Z^{99.9}$, $A_p^{99}$, p99 or p99.9 gradient norm, maximum-coordinate gradient tail, gradient-to-parameter norm ratio, gradient clipping frequency, auxiliary-to-CE gradient ratio, and auxiliary-to-CE cosine. Gradient-norm reports should also specify the logit convention and the pre-clip or post-clip convention. For tied embeddings, output-path and input-path auxiliary gradient norms make the transported update directly visible in diagnostic batches.

For fused or low-precision losses, forward loss equality is most informative when paired with backward-source audits. We recommend auditing the backward source against a high-precision reference. Reports should include relative gradient errors and cosines for both the total CE+Z logit gradient and the Z-loss source. The audit is strongest when the evaluation uses target-model logits from the relevant corpus distribution as well as random tensors, because the error depends on logit scale and distribution. When adaptive optimizers such as Adam or AdamW are used \citep{kingma2015adam,loshchilov2019decoupled}, at least one optimizer-facing check makes the transported effect visible. Examples include selected-parameter or whole-model gradient error and Adam-direction error after backpropagating the approximate source through the model graph on corpus batches. Informative transport and numerical reports should first identify nonfinite events through nonfinite-microbatch rate, nonfinite-gradient rate, and separate counts for NaN and Inf when such events occur. The reports should then describe finite-gradient behavior through finite-only gradient tails, thresholded gradient-event rates, maximum-coordinate gradient tails, and static loss-scale overflow margins when gradients may be stored in fp16. Finally, the reports should include optimizer-memory and loss-tail endpoints, such as Adam second-moment tails, Adam update-state tails, and loss-spike or training-loss-tail endpoints. For MoE routers, nominal $\lambda_R$ is most interpretable with the reduction convention, top-$k$, token-layer denominator, absolute per-decision coefficient, and scale relative to the token-layer mean. When capacity is enforced, reports should include the capacity policy, dropped-route fraction, expert load CV, router entropy, and router correction norm. These details make capacity-limited routing and dropless sparse kernels easier to compare \citep{gale2023megablocks}.

\newpage
\subsubsection{Limitations}

The experiments validate the proposed backward-transport framework across two text corpora, several scales in the GPT-2 and Pythia families, matched continued pretraining, implementation and precision audits, optimizer endpoints, and controlled top-$k$ MoE training. This controlled design isolates source and transport effects rather than reproducing every production fused kernel, expert-parallel configuration, or frontier-scale pretraining setup. These choices define the scope of the empirical validation; the underlying softmax-source identities apply more broadly. As the results show, intervention choice remains model- and regime-dependent and should be guided jointly by validation quality and transport diagnostics. Extending the evaluation to additional production stacks, model families, and training scales is a natural next step.

\end{document}